\documentclass[11pt]{article}
\usepackage[OT1]{fontenc}

\usepackage{times}
\usepackage{amsmath,amsthm, amssymb}
\usepackage{algorithm}
\usepackage{algorithmic}
\usepackage{fullpage}
\usepackage{color}
\usepackage{natbib}
\usepackage{smile}
\usepackage{authblk}
\usepackage{caption}

\def\vol{\mathrm{Vol}}

\def\LP{\text{LP}}
\def\OptLP{\text{Opt-LP}}
\def\BS{\text{BS}}

\title{\huge{Learning to Incentivize  Information Acquisition: \\ Proper Scoring Rules Meet Principal-Agent Model} }

\author[1]{Siyu Chen \thanks{Email: siyu.chen.sc3226@yale.edu}}

\author[2]{Jibang Wu \thanks{ Email: wujibang@uchicago.edu}}

\author[3]{Yifan Wu \thanks{Email: yifan.wu@u.northwestern.edu}}

\author[1]{Zhuoran Yang \thanks{Email: zhuoran.yang@yale.edu}}

\affil[1]{
\small
\textit{Department of Statistics and Data Science, Yale University}}
\affil[2]{
\textit{Department of Computer Science, University of Chicago}}
\affil[3]{
\textit{Department of Computer Science, Northwestern University}}

\date{}

\begin{document}
\maketitle
\setlength\abovedisplayskip{4pt}
\setlength\belowdisplayskip{4pt}
\vspace{-30pt}
\begin{abstract}
We study the incentivized information acquisition problem, where a principal hires an agent to gather information on her behalf.
Such a problem is modeled as a Stackelberg game between the principal and the agent, where the principal announces a scoring rule that specifies the payment, and then the agent then chooses an effort level that maximizes her own profit and reports the information.
We study the online setting of such a problem from the principal's perspective, i.e., designing the optimal scoring rule by repeatedly interacting with the strategic agent. We design a provably sample efficient algorithm that tailors the UCB algorithm \citep{auer2002finite} to our model, which achieves a sublinear $  T^{2/3}$-regret after $T$ iterations. 
Our algorithm features a delicate estimation procedure for the optimal profit of the principal, and a conservative correction scheme that ensures the desired agent's actions are incentivized. Furthermore, a  key feature of our regret bound is that it is independent of the number of states of the environment.  
\end{abstract}


\section{Introduction} \label{sec:intro} 

Delegated information acquisition is a widely popular situation where one party (known as the principal) wants to acquire some information that assists decision-making, and thus hires another party (known as the agent) to gather information on its behalf.
Consider a portfolio manager who aims to learn the potential of a company in order to decide whether to invest in its stock. The manager hires an analyst, who spends some effort to conduct the research, hands in a report to the manager, and gets payment according to the report she writes.
Based on the report, the manager makes the investment decision and earns profit if the stock rises.
The level of effort the analyst puts into the research affects the quality of the information gathered, i.e., the report, and her own cost incurred in conducting the research.
As a result, a rational analyst would choose an effort level that maximizes her own profit -- the difference between the payment and the cost.
Whereas the manager also wants to maximize its own profit in expectation, which is given by the expected gain from the investment, subtracted by the payment.
Knowing that the analyst is rational, the goal of the manager is to design a payment rule that incentivizes the analyst to spend a proper amount of effort, such that the acquired information leads to the investment decision that maximizes the portfolio manager's profit.

Mathematically, such a problem can be modeled under the principal-agent model \citep{laffont2009theory}, where the principal wants to know the state $\omega $  of a stochastic environment,
 and the information acquired by the agent is a distribution $\hat \sigma $ over the state space $\Omega$, also known as the reported belief. 
The game between the principal and the agent is as follows. 
At the beginning, the state is not realized and unknown to both parties. 
The principal 
chooses a scoring rule $S$ \citep{savage1971elicitation, gneiting2007strictly} as the payment rule and  presents it to the agent. 
The agent chooses among $K$  effort levels by selecting an action $b_k\in \{b_1,\ldots, b_K\} $   at a cost $c_k$. Then $b_k$  determines
the joint distribution $p(\omega, o \given b_k)$ of the state $\omega$ and an observation $o$.  
Based on such a conditional distribution  and the realized value of $o$, the agent reports $\hat \sigma$ to the principal. 
Based on such acquired information, the principal chooses an action $a \in \mathcal{A}$. 
Finally, the state $\omega$ is revealed, the principal receives utility $u(a, \omega)$ and pays the agent $S(\hat \sigma, \omega)$. 
Here, the payment to the agent is determined by the  scoring rule $S$, which quantifies the value of the   reported belief by comparing it  with the realized state. 

More generally, our model is a general Stackelberg game with information asymmetry \citep{von1952theory, mas1995micro}, where the agent knows the distribution of the state while the principal does not. 
The principal first  announces a scoring rule $S$. Then the principal chooses an effort level $b_k$ which maximizes her own profit and reports a  belief $\hat \sigma$. That is, $b_k$ is the best response   of the agent to $S$. 
The expected profits of both the principal and the agents are functions of $S$, $b_k$, and $\hat \sigma$. 
From such a perspective, designing the  scoring rule that optimally elicits information is equivalent to finding the strong Stackelberg equilibrium of such a Stackelberg game.

In this work, we focus on the online setting of such a principal-agent model.
In particular, we aim to answer: 
\begin{center}
     \emph{From the perspective of the principal, 
     how to learn the optimal scoring rule \\ by interacting with a strategic agent?} 
\end{center}

The online setting comes with a few challenges. 
First, as the agent is strategic, the reported belief $\hat \sigma$ might be untruthful, or even arbitrary. 
Second, similar to other online learning problems such as bandits \citep{lattimore2020bandit}, we need to explore the stochastic  environment.
More importantly, in our problem, the distribution of the state is determined by the action $b_k$ of the agent, which is beyond the control of the principal. Thus, any successful learning algorithm needs to execute scoring rules that incentivize the agent to explore her action space. 
Third, both the cost $c_k$ and the distribution $p(\omega, o \given  b_k) $ are unknown to the principal. 
In other words, the profit functions of both the principal and the agent are unknown and needs to be estimated from the online data. In particular, this implies that the best response of the agent, as a function of the principal's scoring rule, is unknown. To find principal's optimal scoring rule, we need to know how to incentivize the agent to choose the most favorable $b_k$ for principal, which requires learning the best response function.

We tackle these challenges by introducing a novel algorithm, OSRL-UCB, which leverages  proper scoring rules \citep{gneiting2007strictly}, the principle of  optimism in the face of uncertainty \citep{auer2002finite, lattimore2020bandit}, and the particular constrained optimization formulation of our principal-agent model. 
In particular, to elicit truthful information, we prove a revelation principle~\cite{myerson1979incentive} that shows that it suffices to only focus on the class of proper scoring rules (Lemma \ref{lem:revelation}). 
Then we show that the principal's profit maximization problem can be written as a $K$-armed bandit problem, where the reward of each arm $h^{k,*}$ is the optimal profit  of the principal when the best response of the agent is fixed to $b_k$ (Equation \eqref{eq:bandit problem}). The value of $h^{k, *}$ is determined by the optimal objective of a constrained optimization problem --- finding a proper scoring rule $S^k$ that maximizes the principal's profit, subject to the constraint that the best response to $S^k$ is $b_k$.
Furthermore, we show that  $h^{k, *}$  can be equivalently written as the optimal value of a 
 constrained optimization linear program (LP) involves the unknown pairwise cost  differences and the distribution of the truthful belief induced by $b_k$ (Equation \eqref{eq:cV LP}). 
Following the optimism principle, we aim to construct upper confidence bounds (UCB) of each $h^{k,*}$ and incentive the agent to pick the action that maximizes the UCB. 
To this end, we construct confidence sets  for the pairwise cost differences and belief distributions  based on the online data, and obtain a UCB of $h^{k, *}$ by solving an optimistic variant of the LP by replacing the unknown quantities by elements in the corresponding confidence sets (Equation \eqref{eq:LP_k}). Furthermore, the optimal solution to such an optimistic LP, which is a scoring rule, might violate the condition that its  best response is $b_k$. To remedy this, we devise a conservative modification scheme which simultaneously guarantees desired best response and  optimism with high probability. 
Finally, we prove that the proposed algorithm achieves a $\tilde \cO  (K^2   \mathcal{C}_{\mathcal{O}}  \cdot T ^{2/3})$ sublinear   regret upper bound  after $T$ rounds of interactions. 
Here $K$ is the number of effort levels of the agent and $\mathcal{C}_{\mathcal{O}}   $ is the number of all possible observations, and $  \tilde \cO$ omits logarithmic terms.
A key feature of our regret bound is that it is independent of the number of states of the environment.  
\subsection{Related Work}\label{sec:related-work-full}

\paragraph{Optimal Scoring Rules.} This paper builds on a set of literature on optimizing scoring rules. Several papers consider the model with multiple levels of effort. \citet{neyman2021binary} design  outcome-optimal scoring rule under a binary-state model, and for integral levels of effort, where effort levels represent the number of samples drawn and are informationally ordered. \citet{HSLW-22} design  effort-optimal scoring rule, under a multi-dimensional binary-state model, and each dimension of the effort corresponds to one of the independent multi-dimensional state. 
In contrast, our paper considers a general state space, with multiple levels of effort not necessarily ordered or independent. 
\citet{li2022optimization,10.1145/3490486.3538261, CY-21} design optimal scoring rule for a binary effort model, which is different to our multiple-effort-level model.
\citet{oesterheld2020minimum} design regret-optimal scoring rule for multiple agents in a single round when the information structure is unknown to the principal, while our model only has one agent, and our learning algorithm achieves diminishing  regret over multiple rounds. 
Also, all the papers mentioned above model the state as exogenously given, while in our paper, the prior of the state can potentially be affected by agent's endogenous action. 

\paragraph{The Principal-Agent Problem.} Our model of information acquisition can be viewed as a class of principal-agent problem, which has been established as a crucial branch of economics known as the contract theory \citep{grossman1992analysis, smith2004contract, laffont2009theory}. 
Driven by an accelerating trend of contract-based markets deployed to Internet-based applications, the principal-agent problem recently started to receive a surging interest especially from the computer science community \citep{dutting2019simple, dutting2021complexity, guruganesh2021contracts, alon2021contracts, castiglioni2021bayesian, castiglioni2022designing}.
As pointed out by \citet{alon2021contracts}, this includes online markets for crowdsourcing, sponsored content creation, affiliate marketing, freelancing and etc. The economic value of these markets is substantial and the role of data and computation is pivotal. Different from the classic contract design problems, we focus on the design of contracts (i.e., scoring rules) that optimally elicit the information acquired by the agents at some cost. 

\paragraph{Online Learning in Strategic Environment.}  More broadly, our work add to the literature on   online learning  in strategic environments, which has gained popularity in recent years. 
In particular, the online learner's utility at each round is determined by both her own action and the strategic response(s) of other player(s) in certain repeated game, and the typical goal of this learner is to find her optimal strategy under some equilibrium
These repeated games are adopted from the influential economic models including, but not limited to, the Stackelberg (security) game~\citep{marecki2012playing, balcan2015commitment, haghtalab2022learning},
 auction design~\cite{amin2013learning, feng2018learning, golrezaei2019dynamic, guo2022no}, 
 matching~\citep{jagadeesan2021learning},
 contract design~\cite{zhu2022sample}, 
 Bayesian persuasion~\citep{castiglioni2020online, 10.1145/3465456.3467593, wu2022sequential}. Our model can be viewed as a generalized information elicitation problem in an online learning setup. To our best knowledge, there is no previous work that considered any similar learning problem. In addition, we remark that, in many of these existing works, e.g.,~\citet{balcan2015commitment,guo2022no,wu2022sequential}, the learner is assumed to have sufficient knowledge about the other strategic player(s), and her uncertainty is regarding the environment or her own utility. Assumptions of such kind can significantly simplify the problem into the standard online learning problems, once the learner can almost predict the best response of other player(s). Our work, however, does not make any of such assumption and the most challenging part of our learning algorithm design is indeed to ensure the desired agent response under uncertainty.
 \citet{camara2020mechanisms} considers the online mechanism design problem, where both the principal and the agent may learn over time from the state history. However, our paper assumes the agent is myopic in each round. 
 


\section{The Information Acquisition Model}\label{sec:model}
In this section, we introduce the problem of optimally acquiring information under the principal-agent framework and formulate the problem as a Stackelberg game. 
Following the revelation principle \citep{myerson1979incentive}, we simplify the problem by restricting to the class of proper scoring rules, which enables efficient online learning.

\subsection{Basics of Information Acquisition}

To formulate the problem of optimally acquiring information under the principal-agent framework, we consider a stochastic environment with a principal and an almighty agent. 
At the $t$-th round, there is a hidden state $\omega_t\in\Omega$ that will affect the principal's utility, but is unknown to both the agent and the principal until the end of this round.
To elicit refined \emph{information}, i.e., the agent's belief of the hidden state,
the principal moves first and offers a scoring rule to the agent, 
based on  which the agent receives a payment according to the quality of her reported belief. 
The agent is allowed to choose an action from her finite action space $\cB$ with some cost, obtain an observation related to the hidden state, and generate a report on her belief, which puts the principal in a better position to make a decision.
In the end, the hidden state is revealed, and a utility is generated for the principal, who then pays the agent based on the scoring rule. 
For any $t\geq 0$, in the $t$-th round, the interactions between the principal and the agent are as follows. 

\begin{mdframed}[style=box]
\textbf{Information acquisition via scoring rule}

\vspace{5pt}
\noindent
At the $t$-th round, the principal and the agent play as the following:
\begin{itemize}[noitemsep, topsep=3pt]
    \item[1.] The principal commits to a scoring rule~$S_t:\Delta(\Omega)\times\Omega\rightarrow\RR_+$, where $\Delta(\Omega)$ is the space of distributions over $\Omega$.
    \item[2.] Based on $S_t$, the agent chooses an action $b_{k_t}\in\cB$ indexed by $k_t$ and bears a cost $c_{k_t}\ge 0$. The action $b_{k_t}$ can be observed by the principal.
    \item[3.] The stochastic environment then selects a hidden state $\omega_t\in\Omega$ and emits an observation $o_t\in\cO$ only for the agent according to $p(\omega_t, o_t\given b_{k_t})$. The hidden state $\omega_t$ is unknown to both the agent and the principal at this moment.
    \item[4.] 
    The agent reports a belief $\hat\sigma_t\in\Delta(\Omega)$ about the hidden state  to the principal.
    \item[5.] The principal makes a decision $a_t\in\cA$ based on $(\hat\sigma_t, k_t, S_t)$.
    \item[6.] In the end, the hidden state $\omega_t$ is revealed. The principal obtains her utility $u(a_t, \omega_t)$ and pays the agent by $S_t(\hat\sigma_t, \omega_t)$.
\end{itemize}
\end{mdframed}
Here, the scoring rule $\cS_t$ is a payment rule that depends on the agent's report $\hat\sigma_t$ and the true state $\omega_t$.
Let $\sS$ be the class of scoring rules with bounded norm $\nbr{S}_\infty\le B_S$.
In the sequel, we assume the principal picks $\cS_t$ from $\sS$.
We assume that the reward function $u:\cA\times\Omega\rightarrow \RR$ also has a bounded norm $\nbr{u}_\infty\le B_u$.
We consider the agent's action set $\cB$ and the observation set $\cO$ to be finite.
Specifically, the agent's action space $\cB=\{b_1,\dots,b_K\}$ has $K$ actions.
In the sequel, we also use the action index $k_t$ to represent the agent's action.
The agent's policy for choosing her action $k_t$, her report $\hat\sigma_t$, and the principal's policy for choosing her action $a_t$ will be introduced shortly after.

Notably, our modeling  captures the endogenous effect that the agent's action choice may influence the environment state. 
This is more general than assuming the state is exogenous, and captures the real-world situations that the act of information acquisition, e.g., market investigation,  affects the stochastic environment. Consider the example of a portfolio manager and financial analyst introduced in \S\ref{sec:intro}. The report written by the analyst about a particular stock, when released to the public, may generate considerable impact and affect the stock price \citep{lui2012equity}. 
 

\paragraph{Information Structure.} 
In the remaining part of this subsection, we ignore the subscript $t$ for a while.
In this information acquisition process, we assume the agent is almighty that has full knowledge of the \emph{information structure}, i.e., each action's cost $c_k$ and the generating process $p(\omega, o\given b_k)$ for the hidden state and the observation under action $b_k$.
Therefore, after obtaining the observation $o$, the agent is able to refine her belief $\sigma\in\Delta(\Omega)$ of  the hidden state via the Bayesian rule 
$
\sigma(\omega)\defeq p(\omega\given o, b_k)=p(\omega, o\given b_{k})/p(o\given b_{k}).
$
Note that $\sigma$ is a random measure mapping from the observation space to the probability space $\Delta(\Omega)$ and captures the randomness in $o$. 
Let $\Sigma\subset \Delta(\Omega)$ be the support of $\sigma$. Define $M$ as the cardinality of $\Sigma$ and it follows from the discreteness of $\cB$ and $\cO$ that $M\le K\times C_{\cO}$, where $C_{\cO}$ is the cardinality of $\cO$.
Let $q_k(\sigma)\in\Delta(\Sigma)$ be the distribution of $\sigma$ under the agent's action $k\in[K]$.
Since $\sigma$ already captures all the information about the hidden state from the observation, we ignore the observation $o$ and refer to the costs $\{c_k\}_{k\in[K]}$ and the distributions of the belief $\{q_k\}_{k\in[K]}$ as the information structure, which are private to the agent. We summarize all types of information in \Cref{table:info}.

\begin{table*}[h!]
\centering
\captionsetup{width=.9\linewidth}
\begin{tabular}{|c|c|c|c|}
    \hline
    Public Info & Agent's Private Info & Principal's Private Info & Delayed Info\\
    \hline
    {$S_t, k_t, \hat\sigma_t$} & $\{c_k\}_{k\in[K]}, \{q_k\}_{k\in[K]}, o_t, \sigma_t$ & $u$ & $w_t$\\
    \hline
\end{tabular}
\caption{
Table for the information types. Here, \emph{Public Info} refers to the information that is observable to both the principal and the agent at each round, \emph{Agent's Private Info} refers to the information that is private to the agent, \emph{Principal's Private Info} refers to the information that is private to the principal, and \emph{Delayed Info} only contains the hidden state $\omega_t$, which is unobservable when round $t$ proceeds but revealed when round $t$ terminates.} 
\label{table:info}
\end{table*}

\subsection{Information Acquisition via Proper Scoring Rules} 
We start with the observation that the information acquisition process can be formulated as a general Stackelberg game.
Let $\mu:\sS\rightarrow [K]$ and $\nu:\sS\times\Sigma\times[K]\rightarrow \Delta(\Omega)$  be the agent's policy for action selection  and  belief reporting, respectively.
Here, the reporting policy $\nu$ depends on $\sigma$ instead of $o$ since $\sigma$ captures all the information about the hidden state.
Given any scoring rule $S$, the agent (as the follower of this game) finds her own action by $k=\mu(S)$,  and reports $\hat\sigma=\nu(S, \sigma, k)$ that maximizes her own expected profit, i.e., 
$ g(\mu,\nu; S) \defeq\EE [ S\rbr{\nu(S, \sigma, \mu(S)), \omega} - c_{\mu(S)}],$ where the expectation is taken over the randomness of $\omega, \sigma$ with respect to $q_{\mu(S)}$. 
Let $\iota:\sS\times\Delta(\Omega)\times[K]\rightarrow \cA$ be the principal's decision policy. 
Hence, the principal (as the leader of this game) is to find the  optimal scoring rule $S$ and the best decision policy $\iota$ that maximize  her own expected profit given the agent's best response $(\mu^*,\nu^*)$, i.e., 
$
h(\mu^*, \nu^*; S, \iota) \defeq \EE [ u ( \iota\big(S, \nu^*(S, \sigma, \mu^*(S)), \mu^*(S)\big), \omega) - S(\nu^*(S, \sigma, \mu^*(S)), \omega) ],
$
where the expectation is taken over the randomness of $\omega, \sigma$ with respect to $q_{\mu^*(S)}$.
The optimal leader strategies are known as the strong Stackelberg equilibria that can be formulated as solutions of a bilevel  optimization problem~\cite{conitzer2016stackelberg}, i.e.,
\begin{equation}\label{eq:stackelberg-1-concise}
\max_{\iota, S\in\sS} h(\mu^*, \nu^*; S, \iota),\quad \ \mathrm{s.t.}\ (\mu^*, \nu^*)\in\argmax_{\mu, \nu}  g(\mu,\nu; S).
\end{equation}
However, the bilevel optimization in  \eqref{eq:stackelberg-1-concise}  is  computationally intractable, since the agent's action space (particularly, the space of reporting scheme $\nu$) is intractable.  



To resolve such an issue,
in the following, we establish a revelation principle result that guarantees that, without loss of generality,  it suffices to only focus on the case where the agent is truthful.  
Specifically, there is a well-known class of scoring rules (in Definition \ref{def:PSR}) that characterizes all scoring rules, under which   truthfully reporting is in the agent's best interest~\cite{savage1971elicitation}. 
\begin{definition}[Proper scoring rule]\label{def:PSR}
    A scoring rule $S:\Delta(\Omega)\times\Omega\rightarrow \RR_+$ is proper if, for any belief $\sigma\in\Delta(\Omega)$ and any reported posterior $\hat\sigma\in\Delta(\Omega)$,  we have $\EE_{\omega\sim \sigma}S(\hat\sigma, \omega)\le \EE_{\omega\sim \sigma}S(\sigma, \omega)$. In addition, if the inequality holds strictly for any $\hat\sigma\neq \sigma$, the scoring rule $S$ is strictly proper. 
\end{definition}
By definition, reporting the true belief maximizes the payment for the agent. 
Following the revelation principle~\cite{myerson1979incentive}, we can argue that any equilibrium with possibly untruthful report of belief can be implemented by an equilibrium with truthful report of belief. This means that the principal's optimal scoring rules can be assumed  proper without loss of generality, and we can thereby restrict the reporting scheme $\nu$ to the truthful ones --- we state the revelation principle for the information acquisition model in the following lemma and defer its proof to \S\ref{sec:proper scoring rule}. 
\begin{lemma}[Revelation principle]\label{lem:revelation}
    There exists a proper scoring rule $S^*$ that is an optimal solution to \eqref{eq:stackelberg-1-concise}.
\end{lemma}

In the sequel, we let $\cS$ denote the class of proper scoring rules with bounded norm $\nbr{S}_\infty\le B_S$. When $S\in\cS$, the agent's best report scheme is $v^*(S, \sigma)=\sigma$ since being truth-telling maximizes the payment, and the principal's best decision policy $\iota^*$ can be simplified to $a^*(\sigma):=\iota^*(S, \sigma, k)$ since $S$ and $k$ adds no information to the hidden state given $\sigma$. 
Using the notations $u(\sigma)=\EE_{\omega\sim\sigma}u(a^*(\sigma), \omega)$, $S(\sigma)=\EE_{\omega\sim\sigma}S(\sigma, \omega)$,  and $k^*(S)=\mu^*(S)$, we can transform the optimization program in \eqref{eq:stackelberg-1-concise} into
\begin{align}\label{eq:stackelberg-2}
&\max_{S\in\cS}\quad \EE_{\sigma\sim q_{k^*(S)}} \sbr{u(\sigma) - S(\sigma)},\\
&\mathrm{s.t.}\quad k^*(S)\in\argmax_{k\in[K]} \EE_{\sigma\sim q_{k}} S\rbr{\sigma} - c_{k}, \notag 
\end{align}
where we will denote the agent's utility function as $g(k,S)\defeq \EE_{\sigma\sim q_{k}} S\rbr{\sigma} - c_{k}$ and the principal's utility function under the agent's best response $k^*(S)$ as $h(S) \defeq \EE_{\sigma\sim q_{k^*(S)}} \sbr{u(\sigma) - S(\sigma)}$. This optimization program can be solved efficiently, e.g., by solving for  the optimal proper scoring rule that induces each of the agent's actions as the best response.  
And we will revisit \eqref{eq:stackelberg-2} in the online learning process where the information structure, i.e., $q_k$ and $c_k$, is unknown to the principal. 

\paragraph{Contract Design as a Special Case of Scoring Rule Design.}
We also remark that our model is fully capable of modeling the standard contract designing problem.
If the information structure is full-revealing (e.g., $\sigma$ is a point belief), our model with the endogenous states degenerates to a standard contract design problem, where the scoring rule $S$ becomes a contract as a mapping from (truthfully reported) outcome to payment. We defer the details to \S\Cref{sec:related to contract}.

\subsection{Online Learning to Acquire Information}
We now formalize the online learning problem of solving the optimal scoring rule for information acquisition. We consider the situation where the principal only has knowledge of her utility function $u$ 
\footnote{Since the utility and the hidden state are both known to the principal at the end of each round, if the agent is truth-telling about her belief, the utility function can be efficiently learned. To simplify our discussion, we consider $u$ to be known by the principal.} and is able to observe the agent's action $b_{k}$
\footnote{In cases where the principal cannot observe the agent's action, there are still ways to distinguish different actions. For instance, when $q_k$ has different support for different $k$ and the agent is truth-telling about her belief under proper scoring rules, the principal is able to learn the support for a particular agent's action by repeating the same scoring rule multiple times. The next time the agent chooses the same action, the principal will be aware. But in general, learning the optimal scoring rule without observing the agent's action still remains a hard problem.}. In Table \ref{table:info}, we summarize all the information types discussed above together with their availability. 
Given $H_{t-1}=\cbr{(S_\tau, k_\tau, \sigma_\tau, \omega_\tau)}_{\tau\in[t-1]}\in\cH_{t-1}$ as the history observed by the principal before round $t$, the principal is able to deploy a policy for the next round's scoring rule $\pi_{t}:\cH_{t-1}\rightarrow \cS$.
Hence, the data generating process is described as the following,
\begin{align}\label{eq:data generating}
    p^\pi(S_t, k_t, \sigma_t, \omega_t\given H_{t-1}) &= \ind\rbr{S_t=\pi_t(H_{t-1}), k_t=k^*(S_t)}\cdot q_{k_t}(\sigma_t) \cdot \sigma_t(\omega_t).
\end{align}
The regret for the online policy $\pi=\cbr{\pi_t}_{t\in[T]}$ is defined as,
\begin{align*}
    &\Reg^\pi(T)
    \defeq T\cdot \max_{S\in\cS} h(S) - \EE_\pi\sbr{\sum_{t=1}^T u(a^*(\sigma_t), \omega_t) - S_t(\sigma_t, \omega_t)}, 
\end{align*}
where the expectation is taken with respect to the data generating process. We aim to develop an online policy $\pi_t$ that learns the optimal scoring rule with small regret. 

\section{The OSRL-UCB Algorithm}
In this section, we introduce the algorithm for Online Scoring Rule Learning with Upper Confidence Bound (OSRL-UCB). We begin with an overview of the algorithm in \S\ref{sec:algo overview} and introduce an action-informed oracle that is necessary for the algorithm in \S\ref{sec:oracle}. In \S\ref{sec:UCB}, we give a detailed description of the OSRL-UCB algorithm. To simplify notation, we let $k$ to denote $b_k$ in the sequel.

\subsection{Algorithm Overview}\label{sec:algo overview}
Define $\cV_k=\cbr{S\in\cS\given g(k, S)\ge g(k', S), \forall k'\in[K]}$ as the $k$-th section in which the agent takes action $b_k$ as her best response. The principal's optimization problem \eqref{eq:stackelberg-2} can be reformulated as
\begin{align}\label{eq:bandit problem}
    \max_{k\in [K]} h^{k,*}\defeq\sup_{S\in\cV_k} \EE_{\sigma\sim q_k}\sbr{u(\sigma)-S(\sigma)}, 
\end{align}
where $h^{k, *}$ is the principal's optimal profit when the agent's best response is $k$.
Let $S^{k, *}$ be the optimal solution to the inner problem of \eqref{eq:bandit problem}.
If the principal knows the best scoring rule $S^{k, *}$ that achieves $h^{k,*}$, the problem immediately reduces to a multi-arm bandit where $k\in[K]$ is the $K$ arms and $h^{k, *}$ is the reward for arm $k$. Such a problem can thus be handled by the standard UCB algorithm \citep{auer2002finite}. 
However, there are two obstacles: (i) the action region $\cV_k$ is unknown; (ii) the belief distribution $q_k$ is unknown.
Recall the definition of $\cV_k$:
\begin{align*}
    \cV_k=\cbr{S\in\cS\given \inp[]{q_k-q_{k'}}{S}_\Sigma\ge c_k-c_{k'}, \forall k'\in[K]}.
\end{align*}
To identify $\cV_k$, we just need to estimate the belief distribution $q_k$ and the pairwise cost difference defined as $C(i,j)=c_i-c_j$. Specifically, estimator $\hat q=(\hat q_k)_{k\in[K]}$ can be updated by the empirical distribution of $\sigma_t$ such that $k_t=k$ while estimator $\hat C=(\hat C(i,j))_{i,j\in[K]}$ can be updated using the following identity
\begin{align}\label{eq:C difference}
    C(i,j)=\inp[]{q_i-q_j}{S}_{\Sigma}, \quad\forall S\in\cV_i\cap\cV_j, 
\end{align}
where we plug in estimator $\hat q$.
In addition, we must identify a scoring rule $S\in\cV_i\cap\cV_j$ to successfully employ \eqref{eq:C difference}. 
To this end, we employ a binary search method on the convex combination of $S_1,S_2$ such that $k^*(S_1)=i$ and $k^*(S_2)=j$. The details of the binary search are given in \Cref{alg:BS}.
To inform the principal of the $K$ actions and also to guarantee that the principal can find a scoring rule $S$ such that $k^*(S)=i$ for the sake of the binary search, we introduce with examples an action-informed oracle in \S\ref{sec:oracle}, which provides the principal with foreknown scoring rules $\tilde S_1, \tilde S_1,\dots,\tilde S_K$ such that $k^*(\tilde S_i)=i$.

Now that the estimation problem of $q_k$ and $C(i,j)$ is addressed, the inner problem of \eqref{eq:bandit problem} can be solved by the following constrained linear program,
\begin{align}\label{eq:cV LP} 
    \LP_k:\quad &h^{k, *}=\max_{S\in\cS}\quad \inp[]{ q_k}{u-S}_{\Sigma},\\
    &\mathrm{s.t.}\quad\inp[]{q_k- q_{k'}}{S}_\Sigma\ge C(k, k'), \quad\forall k'\in[K]. \notag 
\end{align}
If \eqref{eq:cV LP} is solved, we can reduce the outer problem of \eqref{eq:bandit problem} to a bandit by viewing $[K]$ as the set of arms and the $h^{k, *}$ as the reward for each arm $k\in[K]$. 
To illustrate the method for solving $\LP_k$,
let $\tilde \cQ=\{(\tilde q_k)_{k\in[K]}\}$ and $\tilde \cC=\{(\tilde C_{ij})_{i,j\in[K]}\}$ be the confidence sets for $\hat q$ and $\hat C$ that capture the real $q$ and $C$ with high probability.
To encourage exploration, we follow the principle of optimism and obtain an upper confidence bound for $h^{k, *}$ by solving $\LP_k$ with plugged-in $(\tilde q, \tilde C)$ that maximizes the optimization goal (principal's profit) over the confidence sets $\tilde \cQ\times\tilde \cC$.
This optimistic version of $\LP_k$ is given by 
$\OptLP_k$ in \eqref{eq:LP_k}
, in which we do not explicitly construct these confidence sets, but instead exploit the confidence intervals for estimating the utilities, the payments, and the cost differences using $\hat q$ and $\hat C$. 
With the optimistic reward for each action given by $\OptLP_k$, the algorithm simply chooses action $k_t^*$ that maximizes this reward, and obtain the scoring rule solution $S_t^*$ corresponding to $k_t^*$.

However, there is one remaining issue if we want to deploy $S_t^*$ to incentive the agent to choose action $k_t^*$. The fact that the constraints of \eqref{eq:cV LP} are relaxed in $\OptLP_k$ by using optimism may cause $S_t^*$ to go beyond $\cV_{k_t^*}$.
To address such a problem, we propose to deploy a conservatively adjusted scoring rule $S_t = (1-\alpha_t) S_t^* + \alpha_t \tilde S_{k_t^*}$, which guarantees the agent to choose $k_t^*$ with high probability. By letting $\alpha_t$ to decay with $t$, the conservatively adjusted scoring rule also guarantees optimism.
The algorithm is summarized in Algorithm \ref{alg:UCB} and details of the algorithm is available in \S\ref{sec:UCB}.

\begin{figure}
\makebox[\linewidth]{%
\begin{minipage}{.95\linewidth}
\begin{algorithm}[H]
\begin{algorithmic}[1]
\STATE {\bfseries Input:} {$\cS$, $(\tilde S_{1}, \cdots, \tilde S_{K})$, $T$.}
\WHILE{$t\le T$}
\STATE Estimate the belief distributions $\hat q_k^t$ by \eqref{eq:hat q} with confidence intervals $I_q^t(k)$ by \eqref{eq:I_q};
\STATE Estimate the cost differences $\hat C^t(i,j)$ with confidence intervals $I_c^t(i,j)$ by \eqref{eq:hat C-I_c};
\STATE Solve $\OptLP_k$ in \eqref{eq:LP_k} with the optimal value $h_\LP^t(k)$ and solution $S_{\LP, k}^t$ for $k\in[K]$;
\STATE Choose the best arm $k_t^* \leftarrow \argmax_{k\in[K]} h_\LP^t(k)$ and let  $S_t^*\leftarrow S_{\LP, k_t^*}^t$;
\STATE Announce scoring rule $S_t = \alpha_t\tilde S_{k_t^*} + \rbr{1-\alpha_t} S_t^*$, and get the agent's response $k_t$;
\IF{$k_t\neq k_t^*$}
    \STATE Conduct binary search $\text{BS}(S_t, \tilde S_{k_t^*}, k_t^*, t)$ as specified in \Cref{alg:BS};
\ENDIF
\STATE Move on to a new round $t\leftarrow t+1$;
\ENDWHILE
\end{algorithmic}
\caption{Online Scoring Rule Learning with Upper Confidence Bound (OSRL-UCB)}\label{alg:UCB}
\end{algorithm}
\end{minipage}}
\end{figure}


\subsection{Oracle for Action Awareness}\label{sec:oracle}
We assume that there is an oracle that provides the principal with a foreknown set of scoring rules that induce the agent to pick all her actions. This is a strictly weaker assumption than many existing online learning models in strategic environments, which assume that the principal can predict the best response of the agent, or know some parameters of the agent's utility function.

\begin{assumption}[Action-informed oracle]\label{asp:oracle}
We assume that there is an oracle which comes up with $K$ scoring rules $(\tilde S_1, \cdots, \tilde S_K)$ such that under $\tilde S_k$, the agent's best response is action $k$ for $k\in[K]$. Moreover, for the agent's profit $g(k, S)=\EE_{\sigma\sim q_k} S(\sigma)-c_k$, we assume that there exists $\varepsilon>0$ such that for any $k'\neq k$, we have $g(k, \tilde S_k)-g(k', \tilde S_k) > \varepsilon$.
\end{assumption}
With the oracle in Assumption \ref{asp:oracle}, the principal is initially aware of the $K$ actions that the agent might respond, which can be easily done by trying $(\tilde S_1, \cdots, \tilde S_K)$ one by one. In addition, we assume that $\tilde S_k$ lies within the interior of the $k$-th section $\cV_k=\cbr{S\in\cS\given g(k, S)\ge g(k', S), \forall k'\in[K]}$
and keeps some distance from the boundary of $\cV_k$. We remark that having $\tilde S_k$ away from the boundary is essential to induce desired behavior in the agent with high probability when applying an approximating scoring rule based on $\tilde S_k$ but with some errors. 
Notably, the assumption that the agent has marginal gain $\varepsilon$ by choosing $k$ under $\tilde S_k$ also ensures a minimum radius of section $\cV_k$, which corresponds to the non-degenerate assumption in strategic Stackelberg games \citep{letchford2009learning}.
We first present an impossible result for online learning the optimal proper scoring rule without the oracle.
\begin{lemma}[Impossible result]\label{lem:impossible}
    Suppose number of actions $K\ge 3$ and the number of possible beliefs $M\ge 3$. Without the action-informed oracle, for any online algorithm, there exists an instance such that $\text{Reg}(T) = \Omega(T)$.
\end{lemma}
See \S\ref{prof:impossible} for a construction of the hard instance.
Without the action-informed oracle, any online algorithm inevitably suffers from a linear regret. The intuition behind the hard instance is that without the oracle, any algorithm can fail to locate a scoring rule that  induces the desirable action from the agent.
This happens because the feasible region, i.e., $\cV_k$ in the space of bounded proper scoring rules for this desirable action $k$, can be arbitrarily small. 

We give two examples where an action-informed oracle is achievable. The first example is through random sampling, adapted from \citep{letchford2009learning}. We sample from the class of strongly proper scoring rules $\cS_\beta$, with definition deferred to \Cref{def:strongly proper}.

\begin{example}[Action awareness via random sampling]\label{exp:random sampling}
Let  $d_1=\min_{1\le i< j\le M}\nbr{\sigma_i-\sigma_j}_\infty$ and $d_2=\min_{1\le k<k'\le K}\max_{i\in[M]}\sbr{q_k(i)-q_{k'}(i)}$.  
Since the proper scoring rule class $\cS$ has bounded norm, the $\beta$-strongly proper scoring class also has bounded volume $\vol(\cS_\beta)<\infty$. 
Set $\kappa={d_1^2\beta}/{2}$ and
let $\tilde\cV_k=\{S\in\cS_\beta\given g(k, S)\ge g(k', S)+\kappa, \forall k'\neq k, k'\in[K]\}$. We  suppose $\vol(\tilde\cV_k)\ge \eta \Vol(\cS_\beta)$ for $k\in[K]$. 
Initiate $\cM=\emptyset$ as the candidate set.
Each time, we randomly sample a $\beta$-strongly proper scoring rule $S\in\cS_\beta$ and obtain the agent's best response $k^*(S)$ with respect to $S$. 
Let $e_i(\sigma, \omega)=\ind(\sigma=\sigma_i)$ for $i\in[M]$. By setting $\kappa=d_1^2\cdot\beta/2$, we deploy $\cbr{S-\kappa e_i}_{i\in[M]}$ and see if the agent still responds $k^*(S)$. If so, Let $\sS=\sS\cup \cbr{S}$; if not, reject $S$. 
    After $\cO(M\eta^{-1} K\log K)$ rounds, with high probability, $\sS$  serves as a valid action-informed oracle  with parameter $\epsilon=d_2\cdot \kappa$.
\end{example}
See \S\ref{sec:random sampling proof} for more details.
In Example \ref{exp:random sampling}, we ensure a set of scoring rules that successfully induces each action by randomly sampling in $\cS$ for up to $\cO(MK\log K)$ times.
We next show another example via use of linear contract where the information structure satisfies some special properties. 

\begin{example}[Action awareness via linear scoring rule]\label{exp:linear contract}
Suppose that these $K$ actions are sorted in an increasing order with respect to the cost. Define $u_k = \EE_{\omega\sim\sigma, \sigma\sim q_k}\sbr{u(a^*(\sigma), \omega)}$ and suppose $u_k>0$. We assume the marginal information gain is strictly decaying, i.e., there exists a $\epsilon>0$ such that 
    \begin{gather*}
        \frac{u_{K}-u_{K-1}}{c_{K}-c_{K-1}} >\epsilon, \quad \text{and}
        \quad
        \frac{u_k-u_{k-1}}{c_k-c_{k-1}} - \frac{u_{K}-u_k}{c_{K}-c_k} >\epsilon, \quad \forall k=2,\cdots, K-1.
    \end{gather*}
    Moreover, we assume that $(u_2-u_1)/(c_2-c_1)\le b$. The principal sets the scoring rule as $S(\sigma, \omega)=\lambda u(a^*(\sigma), \omega)$ and conducts binary search on $\lambda\in[0, 2/\epsilon]$. Specifically, the binary searches are iteratively conducted on all the segments $(\lambda_1, \lambda_2)$ with $k^*(\lambda_1 u) \neq k^*(\lambda_2 u)$, where $\lambda_1, \lambda_2$ are neighboring points on $[0, 1/\epsilon]$ that are previously searched. With the maximal searching depth $m= \ceil{\log_2(2b(b-\epsilon)/\epsilon^2)}$, we can identify all the agent's actions. Suppose that $(\lambda_1^k, \lambda_2^k)$ is the largest segment with $\lambda_1^k$ and $\lambda_2^k$ searched before and $k^*(\lambda_1^k u)=k^*(\lambda_1^k u)=k$. By letting $\tilde S_k=(\lambda_1^k + \lambda_2^k)u/2$, we obtain an oracle with $\varepsilon=\epsilon u_1/4b^2$. The procedure takes $\cO(K\log_2(\varepsilon^{-1}))$ rounds.
\end{example}
See \S\ref{sec:linear contract proof} for more details.
In this example, we exploit the power of linear contract \citep{dutting2019simple} to identify all the agent's actions and produce a set of scoring rules based on the principal's utility. Specifically, the assumption of decaying marginal information gain is commonly seen in real world practice, and it further guarantees that all the actions are inducible using linear contracts. To obtain such an oracle, we just need at most $\cO(K\log_2(\varepsilon^{-1}))$ trials, which is more efficient than random sampling.
We remark that these foreknown scoring rules obtained by the oracle do not need to be optimal in each section $\cV_k$.
They can even be obtained through random sampling from $\beta$-strongly proper scoring rules  (See \Cref{exp:random sampling}) for general setting, or discovered in a linear scoring rule class (See \Cref{exp:linear contract}) if the marginal information gain is strictly decaying.

\subsection{OSRL-UCB Algorithm}\label{sec:UCB}
The OSRL-UCB algorithm contains the following four parts: (i) learning the belief distributions; (ii) learning the pairwise cost differences; (iii) solving for the optimal payment/scoring rule in each $\cV_k$ via optimistic linear programming $\OptLP_k$ and choosing the best \say{arm} via UCB; (iv) applying a conservative scoring rule and conduct binary search if $k_t\neq k_t^*$ to refine our estimation on $\cV_k$.
In this subsection, we describe the OSRL-UCB algorithm in detail.

\paragraph{Learning the Belief Distributions.}
Let $n_{k}^t$ denote the total number of times the agent responds action $k$ before round $t$. Then, $q_k$ can be learned empirically as 
\begin{align}\label{eq:hat q}
    \hat q_k^t(\sigma) = \frac{1}{n_{k}^t} \sum_{\tau=1}^{t-1}\ind(\sigma_\tau=\sigma, k_\tau=k), \forall \sigma\in\Sigma.
\end{align}
Following from a standard concentration result in \citet{mardia2020concentration},  we define the confidence interval for $\hat q_k^t$ as
\begin{align}\label{eq:I_q}
    I_{q}^t(k) =  \sqrt{\frac{2\log((K)\cdot 2^M t)}{n_{k}^t}}
\end{align}
We state the concentration result in \Cref{lem:confidence-interval}.
Under $\hat q_k^t$, the estimated payment of scoring rule $S$ if the agent responds action $k$ is $
    \hat v_{S}^t(k) = \inp[] {S(\cdot)} {\hat q_k^t(\cdot)}_\Sigma$.

\paragraph{Learning the Pairwise Cost Differences.}
We next illustrate how to learn the pairwise cost difference. 
For each $\tau<t$ such that $k_\tau = i$ and $j\neq i$, define 
\begin{gather*}
    C_+^t(i, j) = \min_{\tau<t:k_\tau=i}\hat v_{S_\tau}(i) - \hat v_{S_\tau}(j) + B_S \rbr{I_{q}^t(i) + I_{q}^t(j)},\nend
    C_-^t(i,j)=\max_{\tau<t:k_\tau= j} \hat v_{S_\tau}(i) - \hat v_{S_\tau}(j) - B_S \rbr{I_{q}^t(i) + I_{q}^t(j)}.
\end{gather*}
For each pair $(i, j)$, we also define 
\begin{align*}
    \theta^t(i, j) &= \frac{C_+^t(i, j) + C_-^t(i, j)}{2}, 
    \quad
    \varphi^t(i, j) = \sbr{\frac{C_+^t(i, j) - C_-^t(i, j)}{2}}_+.
\end{align*}
Directly using $\theta^t(i, j)$ as the estimation of pairwise cost difference is not the best option for two reasons: (i) For $\theta^t(i,j)$ to be accurate, we need to detect $S_\tau$ such that $S_\tau$ lies on the boundary $\cV_i\cap\cV_j$. However, it can happen that $\cV_i\cap\cV_j=\emptyset$, thus producing a constant error. (ii) Even if $\cV_i\cap\cV_j\neq \emptyset$, finding $S_\tau \in \cV_i\cap\cV_j$ for all $(i,j)$ pairs can be costly and potentially increase the algorithm complexity. 
Instead, we observe that the number of unknown parameters in the cost is at most $K$, thus it suffices to search in a \say{tree} structure. 
Specifically, let $\cG=(\cB, \cE)$ denote the graph where the node set $\cB$ corresponds to the $K$ agent actions and the edge set $\cE$ corresponds to the pairwise cost differences $C(i,j)=c_i-c_j$. 
In addition, we let $\varphi^t(i,j)$ be the \say{length} assigned to edge $C(i,j)$, which corresponds to the uncertainty of using $\theta^t$ to estimate the cost difference. Therefore, the problem of estimating $C(i, j)$ with minimal error becomes finding the \textit{shortest path} between $b_i$ and $b_j$ on the graph $\cG$, which can be efficiently handled by Dijkstra's algorithm or the Bellman-Ford algorithm \citep{ahuja1990faster}. In summary, the cost difference is estimated by
\begin{gather}
    l_{ij}=\text{shortest path}(\cG, i,j), \quad \hat C^t(i, j) = \sum_{(i', j')\in l_{ij}} \theta^t(i', j'), 
    \quad
    I_{c}^{t}(i, j) = \sum_{(i', j')\in l_{ij}} \varphi^t(i', j'), \label{eq:hat C-I_c}
\end{gather}
where $I_c^t(i,j)$ is the confidence interval for the pairwise cost estimator $\hat C^t(i,j)$. 
Moreover, we can easily check that $\hat C^t(i, j)=-\hat C^t(j, i)$ since $l_{ij}$ is the same as $l_{ji}$ and $\theta^t(i, j)=-\theta^t(j, i)$ by definition. The validity of $I_c^t(i,j)$ serving as a confidence interval for $\hat C^t$ is proved in \Cref{lem:C interval}.

\paragraph{Solving for $\OptLP$ and Choosing the Best Arm.}
Note that we have the estimator $\hat q_k^t$ for the belief distributions and the estimator $\hat C^t$ for the pairwise cost differences. We are now able to solve the linear program (LP) given in \eqref{eq:cV LP} for the best scoring rule corresponding to agent action $b_k$. To incorporate the optimism principle, we relax the constraint of \eqref{eq:cV LP} using the confidence interval $I_q^t$, $I_c^t$ obtained before. 
Specifically, for agent action $b_k$, we consider the following optimistic linear program $\OptLP_k$, 
\begin{align}\label{eq:LP_k}
    &\OptLP_k:\quad \max_{S\in\cS} \quad \inp[]{u}{\hat q_k^t}_\Sigma + B_u I_q^t(k) -v, \notag \\
    &\qquad\qquad\quad \mathrm{s.t.}\quad \abr{v-\hat v_{S}^t(k)} \le B_S \cdot I_{q}^t(k), \\
    &v - \hat v_{S}^t(i) \ge \hat C^t(k, i) - \rbr{I_{c}^{t}(k, i) + B_S\cdot I_{q}^t(i)}, \forall i\neq k. \notag 
\end{align}
Here, the optimization goal is to maximize the principal's profit under the agent's best response $b_k$, where we add $B_u I_q^t(k)$ to upper bound the true value with high probability. 
The first constraint actually constructs a confidence interval $B_S I_q^t(k)$ for the payment $v$, while the second constraint relax the initial boundary condition in \eqref{eq:cV LP} using $I_c^t$ and $I_q^t$.
The relaxations in the constraints guarantee that $\OptLP_k$ is optimistic with high probability, as is verified in \Cref{lem:linear-program-bound}. 
Suppose that the optimal value and solution for $\OptLP_k$ are $h_\LP^t(k)$ and $S_{\LP, k}^t$, respectively. 
If $\OptLP_k$ is infeasible, we just let $h_\LP^t(k) = \inp[]{u}{\hat q_k^t}_\Sigma -\EE_{\sigma\sim\hat q_k} \tilde S_k(\sigma) + (B_S + B_u)I_{q}^t(k)$ and $S_k^t=\tilde S_k$.
Next, by viewing the problem as a $K$-arm bandit, we choose the best \say{arm} that maximizes the principal's optimistic expected profit,
$k_t^* = \argmax_{k\in[K]} h_\LP^t(k).  $
For simplicity, we let $S_t^* := S_{\LP, k_t^*}^t$.

\paragraph{Constructing Conservative Scoring Rules.}
However, it happens that we can sometimes be ``overoptimistic'' about the agent's best response. That is, as we relax the boundary constraint in $\OptLP_k$, the agent might not respond action $k_t^*$ under scoring rule $S_t^*$.
This fact suggests that we ought to be \emph{conservative} to a certain degree in the design of scoring rule. In particular, we consider the conservative scoring rule as,
$ S_t = (1-\alpha_t) S_t^* + \alpha_t \tilde S_{k_t^*}.$
The intuition is that since the agent strictly prefers the action $k_t^*$ under the scoring rule $\tilde S_{k_t^*}$, combining $S_t^*$ with $\tilde S_{k_t^*}$ increases the agent's relative preference of choosing action $k_t^*$. 
In \Cref{lem:mistake}, we show that with a choice of $\alpha_t=\cO(t^{-1/3})$, we can guarantee with high probability that the agent would response with action $k_t^*$ under the conservative scoring rule $S_t$. This also means the optimism (reflected by the agent's choice of action $k_t^*$) is guaranteed.

\paragraph{Refining Parameter Estimations.} Once $S_t$ is deployed, we consider two types of outcomes. If the agent responds with action $k_t=k_t^*$, our estimates of agent's decision boundary $\cV_{k_t^*}$ is good enough, and we can simply proceed to the next round. Otherwise, the agent responds with another action $k_t\neq k_t^*$, and we need to improve our estimations on $\cV_{k_t^*}$ by updating $\hat q^t, \hat C^t$ in order to successfully induce the desired action $k_t^*$ in the future. Specifically, due to the special conservative construction of $S_t$, it suffices to update the boundary of $\cV_{k_t^*}$ that lies between $S_t$ and $\tilde S_{k_t^*}$.
To achieve this, we conduct binary search on the segment connecting $\tilde S_{k_t^*}$ and $S_t$ and locate the first switch point where the agent's best response changes from action $k_t^*$ to another action. Details of the binary search algorithm are available in \Cref{alg:BS}.


To be specific, in the process of searching the switching point, there are two useful information that we can utilize: Firstly, for the boundary of $\cV_{k_t^*}$ that lies on the segment $(\tilde S_{k_t^*},S_t)$, we induce the two actions $(i,j)$ that this boundary separates at least once, thus their belief distribution estimator $\hat q_i^t$ and $\hat q_j^t$ will be updated.
Secondly, and more importantly, this switching point $S$ lies near the boundary, we can refine the cost difference by $ \hat C(i,j)\approx\inp[]{\hat q_i^t-\hat q_j^t}{S}_{\Sigma}$ in \eqref{eq:C difference}.
Thus, this binary search can refine both the belief estimator and the cost difference estimator, which leads to a better estimation of $\cV_{k_t^*}$. Notably, such  a binary search procedure achieves sufficient accuracy within logarithmic searching time.


\section{Theoretical Results}

In this subsection, we provide the learning result for the OSRL-UCB algorithm.
\begin{theorem}[Regret for OSRL-UCB]\label{thm:regret}
    Under Assumption \ref{asp:oracle} on the action-informed oracle, with $\alpha_t = \min\{K t^{-1/3}, 1\}$, the   OSRL-UCB algorithm \ref{alg:UCB} has regret
\begin{align*}
    \Reg(T) = \tilde\cO\bigl ( (B_S+B_u)B_S^2\varepsilon^{-2} \cdot K^2 C_\cO\cdot T^{2/3}\big) , 
\end{align*}
where $B_S$ and $B_u$ bound the magnitudes of the scoring rule and the utility function, respectively, $\varepsilon$ is the marginal profit gain given by the action informed oracle, $K$ is the number of the agent's actions, and $C_\cO$ is the cardinality of the observation set.
\vspace{-10pt}
\begin{proof}
We defer the detailed proof to \S\ref{prof: main theorem}. 
\end{proof}
\end{theorem}

Here, we use $\tilde \cO$ to omit logarithmic factors.
The regret depends quadratically on the agent's action number and linearly on the cardinality of the observation set.
Notably, the regret is independent of the size of the hidden state $|\Omega|$. 
In addition, we achieve a $\cO(T^{2/3})$ sublinear rate of regret in terms of the principal's accumulative profit for eliciting information under the scoring rule framework. Such a result is achieved with a mild assumption of an action-informed oracle that provides a set of scoring rules with marginal profit gain for the agent that induce all the agent's actions. 
We do not assume the learner to have sufficient knowledge about the other strategic player(s) in contrast to many existing works \citep{balcan2015commitment,guo2022no,wu2022sequential}. 
In addition, we only assume the principal to have knowledge of her utility and can observe the agent's action choice. For discussion of the these two assumptions, we refer the readers to the footnote in \S\ref{sec:model}.

For the action-informed oracle with a set of foreknown scoring rules, these foreknown scoring rules do not need to be optimal in each section $\cV_k$.
They can even be obtained through random sampling from $\beta$-strongly proper scoring rules  (See \Cref{exp:random sampling}) for general setting, or discovered in a linear scoring rule class (See \Cref{exp:linear contract}) if the marginal information gain is strictly decaying.
We also give the following corollaries that characterize the regret combined with the effort to find an action-informed oracle.
\begin{corollary}[Regret with oracle in \Cref{exp:random sampling}]\label{cor:random sampling}
Let $\tilde\cV_k=\{S\in\cS_\beta\given g(k, S)\ge g(k', S)+\kappa, \forall k'\neq k, k'\in[K]\}$ where $\cS_\beta\in\cS$ is the class of $\beta$-strongly proper scoring rules, and suppose $\vol(\tilde\cV_k)\ge \eta \Vol(\cS_\beta)$ for $k\in[K]$.
Running the oracle acquisition process in Example \ref{exp:random sampling} for $T^\gamma$ rounds before deploying the OSRL-UCB algorithm for $T-T^\gamma$ rounds, the online regret is bounded by
\begin{align*}
    \Reg(T) &= \tilde\cO\big( (d_2 d_1^2 \beta)^{-2} \cdot KM\cdot T^{2/3}\big) 
    +\cO(KT\exp(-T^\gamma\eta/M) + T^\gamma),
\end{align*}
where $d_1=\min_{i\neq j, \forall (i,j)\in[M]}\nbr{\sigma_i-\sigma_j}_\infty$, $d_2=\min_{k'\neq k, (k, k')\in[K]}\max_{i\in[M]}\sbr{q_k(i)-q_{k'}(i)}$, and $M$ is the cardinality of $\Sigma$.
\end{corollary}
And also we characterize the regret for the action-informed oracle obtained by linear scoring rule.
\begin{corollary}[Regret with oracle in \Cref{exp:linear contract}]\label{cor:linear contract}
Suppose the model assumption that the marginal information
gain is strictly decaying in \Cref{exp:linear contract} holds.
By running the oracle acquisition process in \Cref{exp:linear contract} for $\cO(K\log_2(\varepsilon^{-1}))$ rounds and the OSRL-UCB algorithm for the remaining rounds, the online regret is bounded by
\begin{align*}
    \Reg(T) &= \tilde\cO\big(\varepsilon^{-2} \cdot K^2 C_\cO\cdot T^{2/3}\big)+ \cO(K\log_2(\varepsilon^{-1})), 
\end{align*}
where $\varepsilon=\epsilon u_1/4b^2$, and $\epsilon, u_1, b$ are constants defined in \Cref{exp:linear contract}.
\end{corollary}

Corollaries \Cref{cor:random sampling} and \Cref{cor:linear contract} both provide regret bound without any using of oracle. Specifically, \Cref{cor:random sampling} considers a more general framework under the assumption of lower bounded action section volume while \Cref{cor:linear contract} assumes marginal information decay, which is commonly seen in real world practice. Specifically, \Cref{cor:random sampling} shows that by random sampling for $T^\gamma$ rounds where $0<\gamma<2/3$, it suffices for the principal to have $\tilde \cO(T^{2/3})$ regret. In addition, $\gamma$ can be significantly small since the second term diminishes exponentially on $T^\gamma$. 
In addition, \Cref{cor:linear contract} shows that running constant number of additional rounds in the oracle acquisition process does not deteriorate the regret bound.

Following the discussion in \citet{jin2018q}, our algorithm also has PAC guarantee as the following.
\begin{corollary}[PAC guarantee]
For every $\zeta>0$, the OSRL-UCB algorithm with action informed oracle finds a $\zeta$-optimal scoring rule using $\tilde\cO(\varepsilon^{-6} K^6C_\cO^3\zeta^{-3})$ samples.
\end{corollary}

\vspace{-1mm}
\citet{zhu2022sample} provides an $\cO(T^{2/3})$ regret lower bound for the online learning problem towards the optimal contract. Despite that standard contract design is a special case of our model, their setting is different from ours in that the agent may have possibly infinitely many actions. So the regret lower bound of our problem remains an open question. To close this gap, we believe the key question remains to be answered is whether the decision boundary of $\cV_{k^*}$ can be determined efficiently. That is, even if the best action $k^*$ is known, the learner is still unable to solve the optimal scoring rule from $\OptLP_{k^*}$ without enough knowledge about $\cV_{k^*}$. For now, we are able to construct a class of instances where such boundary of $\cV_{k^*}$ can determine with binary search and thus avoid the costly learning of $q_k$ for every $k\in[K]$.  However, it still remains unclear if these efficient search techniques can possibly be generalized to arbitrary instances --- a definitive answer should close up the regret lower and upper bound of this problem. We leave this intriguing direction to  future work.
\vspace{-2mm}
\section{Conclusion}
\vspace{-1mm}
We study the problem of incentivizing information acquisition through proper scoring rules under the principal-agent framework with information asymmetry. We propose the OSRL-UCB algorithm and show that with a mild oracle assumption, it  achieves  a $\cO(K^2 C_\cO T^{2/3})$ sublinear regret. Future direction includes establishing regret lower bound  and extensions to the contextual and dynamic settings.

\newpage
\bibliographystyle{econ}
\bibliography{ref.bib}

\newpage
\appendix

\section{Contract Design as a Special Case of Scoring Rule Design}\label{sec:related to contract}
In this section, we compare the contract design framework to the scoring rule design framework and reduces the standard contract designing problem to a special case of the scoring rule designing problem.
In a contract designing problem, we refer to $\Omega$ as the outcome space. 
Consider the following contract designing problem:
\begin{mdframed}[style=box]
\textbf{Contract designing problem in the principal-agent framework}

\vspace{5pt}
\noindent
At the $t$-th round, the principal and the agent play as the following:
\begin{itemize}[noitemsep, topsep=3pt]
    \item[1.] The principal announces a contract $C_t:\Omega\rightarrow \RR_+$ to the agent.
    \item[2.] Based on $C_t$, the agent chooses an action $b_{k_t}\in\cB$ indexed by $k_t$ and bears a cost $c_{k_t}\ge 0$. The action $b_{k_t}$ can be observed by the principal, but the cost $c_{k_t}$ is private to the agent.
    \item[3.] The stochastic environment then selects an outcome $\omega_t\in\Omega$ according to $p(\omega_t\given b_{k_t})$. The outcome $\omega_t$ is revealed as observation, but the generating process $p(\omega_t\given b_{k_t})$ is private to the agent.
    \item[4.] The principal makes a decision $a_t\in\cA$ based on $\omega_t$.
    \item[5.] In the end, the principal obtains her utility $u(a_t, \omega_t)$ and pays the agent by $C_t( \omega_t)$.
\end{itemize}
\end{mdframed}
The difference between this contract designing problem and the scoring rule designing problem is that $\omega_t$ is revealing, and the agent's action influences the principal's utility only through her action choice without giving any report. 
We remark that we can replace $u(a_t, \omega_t)$ by $u(\omega_t)=u(a^*(\omega_t), \omega_t)$ if the principal knows about the utility function and always takes the best action. In this contract design problem, the agent has an action policy $\pi:\cC\rightarrow [K]$, where $\cC$ is the contract space.
The principal targets at designing the optimal contract that maximizes her profit, i.e., utility minus payment, subject to the agent's best response given by maximizing the agent's profit, i.e., payment minus cost. The Stackelberg game for this contract designing problem can be formulated as
\begin{equation}
\begin{aligned}\label{eq:stackelberg-contract}
&\max_{C\in\cC}\quad \EE_{\omega\sim p(\cdot\given b_{\pi^*(C)})} \sbr{u(\omega) - C(\omega)},\\
&\mathrm{s.t.}\quad \pi^*(C)\in\argmax_{k\in[K]} \EE_{\omega\sim p(\cdot\given b_k)} C\rbr{\omega} - c_{k},
\end{aligned}
\end{equation}


In the sequel, we aim to show in the scoring rule designing problem:
(i) If the hidden state is perfectly revealing, i.e., $o_t=\omega_t$ as the agent's observation after taking her action, there exists a class of scoring rules such that the above contract designing problem is equivalent to the scoring rule designing problem. 
(ii) Using proper scoring rules, the principal's optimal profit under the scoring rule framework is no less than the optimal profit under the contract framework.

To show (i), consider the scoring rule class 
$$
\sS^C=\cbr{S\in\sS\given S(\hat\sigma, \omega)=\ind(\hat\sigma=e_\omega)\cdot C(\omega), \forall C\in\cC}, 
$$
where $e_{\omega'}(\omega)=\ind(\omega=\omega')\in\Delta(\Omega)$.
Even though $S\in\sS^C$ might not be a proper scoring rule, the agent will always be truth-telling, i.e., $\hat\sigma=e_\omega$, since only by telling the truth can she gains nonzero payment.
Therefore, this hidden state $\omega_t$ is also revealed to the principal through the agent's report.
The Stackelberg game in \eqref{eq:stackelberg-2} in the scoring rule problem can therefore be written as,
\begin{equation}
\begin{aligned}\label{eq:stackelberg-scoring}
&\max_{S\in\sS^C}\quad \EE_{\omega\sim p(\omega\given b_{k^*(S)})} \sbr{u(\omega) - S(e_\omega, \omega)},\\
&\mathrm{s.t.}\quad k^*(S)\in\argmax_{k\in[K]} \EE_{\omega\sim p(\cdot\given b_{k})} S\rbr{e_\omega, \omega} - c_{k},
\end{aligned}
\end{equation}
By noting that $S(e_\omega, \omega)=C(\omega)$, we have that \eqref{eq:stackelberg-scoring} and \eqref{eq:stackelberg-contract} are actually the same problem.
We thus conclude that the contract designing problem is perfectly reduced to this scoring rule designing problem.

We also remark that even if the hidden state is perfectly revealing, the principal need not be aware in advance. By sticking to a proper scoring rule, the agent always tells the truth. Moreover, using the revelation principle stated in \Cref{lem:revelation}, for any $S\in\sS^C$, there always exists a proper scoring rule $\tilde S\in\cS$ that generates the same expected payment $\EE_{\omega\sim p(\cdot\given b_k)} \tilde S(e_\omega, \omega)=\EE_{\omega\sim p(\cdot\given b_k)} C(\omega)$ for the agent, even though $\tilde S(e_\omega, \omega)$ might not be equal to $C(\omega)$ pointwise. Therefore, statement (ii) is also justified and we conclude that the (proper) scoring rule framework has more power than the contract framework by asking one more question about the agent's belief.

\section{More Details on the Revelation Principle}\label{sec:proper scoring rule}
In this section, we provide a formal argument on the revelation principle in our model. That is, it is without loss of generality to only design the proper scoring rules under which the agent is encouraged to be truth-telling. 
\begin{definition}[Proper scoring rule]
    A scoring rule $S:\Delta(\Omega)\times\Omega\rightarrow \RR_+$ is proper if, for any belief $\sigma\in\Delta(\Omega)$ and any reported posterior $\hat\sigma\in\Delta(\Omega)$,  we have $\EE_{\omega\sim \sigma}S(\hat\sigma, \omega)\le \EE_{\omega\sim \sigma}S(\sigma, \omega)$. In addition, if the inequality holds strictly for any $\hat\sigma\neq \sigma$, the scoring rule $S$ is strictly proper. 
\end{definition}
Let $S$ be a proper scoring rule and fix the agent's policy $\mu(\cdot)$ for action selection. For a reporting scheme $\nu$ and any true belief $\sigma$, it follows from definition \ref{def:PSR} that
\begin{align*}
    g^{\mu,\nu}(S)= \EE_{\omega\sim\sigma} S(\nu(S, \sigma, \mu(S)), \omega) - c_{\mu(S)}\le \EE_{\omega\sim\sigma} S(\sigma, \omega) - c_{\mu(S)}.
\end{align*}
Therefore, the agent's expected payment is maximized by always being truth-telling about her belief under the class of proper scoring rule.
In the following, we let $S(\hat\sigma, \sigma)=\EE_{\omega\sim\sigma}S(\hat\sigma, \omega)$. 
To give an example of proper scoring rules, let us consider the binary hidden state space $\Omega=\{0, 1\}$ where the class of proper scoring rules admits the Schervish representation \citep{gneiting2007strictly}, i.e., $S(p, 1)=G(p) + (1-p) G'(p)$ and $S(p, 0)=G(p)-p G'(p)$ where $p\in[0, 1]$ and $G:[0, 1]\rightarrow\RR_+$ is a convex function.
Intuitively, the expected payment of a proper scoring rule $S$ given belief $\sigma$ and report $p$ is $S(p, \sigma)=G(p)+(\sigma-p)G'(p)$, which corresponds to the supporting line of $G$ at $p$.
In this example, the convexity of $G$ guarantees that $S(p, \sigma)=G(p)+(\sigma-p)G'(p)\le G(\sigma)=S(\sigma, \sigma)$.
Moreover, the next observation in Lemma \ref{lem:revelation} suggests that restricting to the class of proper scoring rules does not incur any loss of generality for the principal's purpose.
\begin{lemma}[Restatement of the revelation principle]
    There exists a proper scoring rule $S^*$ that is an optimal solution to \eqref{eq:stackelberg-1-concise} if the agent is truth-telling under any proper scoring rule.
    \begin{proof}
        We first prove that for any scoring rule $S$ such that $\nbr{S}_\infty\le B_S$, there always exists a \emph{proper} scoring rule $S'(\hat\sigma, \omega)=S(\nu^*(S, \hat\sigma, k), \omega)$ such that they make the same payment to the agent for any $\sigma\in\Delta(\Omega)$ and any agent's action choice. To prove that $S'$ is a proper scoring rule with norm bounded by $B_S$, we have $B_S\ge S(\nu^*(S, \hat\sigma, k), \omega)=S'(\hat\sigma, \omega)\ge 0$ and $$S'(\hat\sigma, \sigma)= S(\nu^*(S, \hat\sigma, k), \sigma)\le S(\nu^*(S, \sigma, k), \sigma)=S'(\sigma, \sigma).$$
        The fact that $S'$ makes the same payment can be verified by plugging in $\hat\sigma=\sigma$ in the definition of $S'$ since the agent is truth-telling under proper scoring rules, and taking expectation with respect to $\omega\sim\sigma$, i.e., $S'(\sigma, \sigma)=S(\nu^*(S, \sigma, k), \sigma)$, which proves the first part.
        
        Secondly, we prove that encouraging the agent to report the real belief $\sigma$ makes the principal's revenue nondecreasing. Note that
        \begin{align*}
            \max_{\iota} \EE_{\omega\sim\sigma, \sigma\sim q_k}\sbr{u(\iota(S, \nu^*(S, \sigma, k), k), \omega)} \le  \max_{\iota} \EE_{\omega\sim\sigma, \sigma\sim q_k}\sbr{u(\iota(S, \sigma, k), \omega)}=\EE_{\omega\sim\sigma, \sigma\sim q_k}\sbr{u(a^*(\sigma), \omega)},
        \end{align*}
        where $a^*(\sigma)=\argmax_{a\in\cA}\EE_{\omega\sim\sigma} u(a, \omega)$.
        Here, the inequality holds by noting that $\nu^*(S, \sigma, k)$ is a function of $\sigma$, and the equality holds by noting that $\omega\indep (S,k)\given \sigma$.
        To conclude, by choosing $S'$ instead of $S$, the payment is exactly the same while the principal's revenue is nondecreasing. Thus, the principal's profit is nondecreasing by choosing $S'$ and there must exist a proper scoring rule that is also an optimal scoring rule. 
    \end{proof}
\end{lemma}

Following Lemma \ref{lem:revelation}, the principal's optimal scoring rule lies within the class of proper scoring rules $\cS$ with bounded norm $\nbr{S}_\infty\le B_S$.
One concern about the use of proper scoring rules is that being truth-telling might not be the unique maximizer to the agent's utility. However, we note that the class of proper scoring rules is a convex hull with strictly proper scoring rules as the interior. Thus, adding an infinitesimal portion of a \emph{strictly proper} scoring rule to any \emph{proper} scoring rule always yields a \emph{strictly proper} scoring rule. In this sense, we can safely make the assumption that the agent always reports the true posterior under a proper scoring rule.
The following table summarizes all the different information types in our model.

\section{More Details on the Binary Search Algorithm}
In this section, we give a summary of the binary seach algorithm as follows, 
\begin{figure}
\makebox[\linewidth]{%
\begin{minipage}{\linewidth}
\begin{algorithm}[H]
\begin{algorithmic}[1]
\STATE {\bfseries Input:} $S_0, S_1, k^*(S_1), t, k_t, \{I_q^t(k)\}_{k\in[K]}$;
\STATE {\bfseries Output:} $t$;
\IF{$k_t = k^*(S_1)$}
    \STATE Break the binary search algorithm;
\ENDIF
\STATE Initiate $\lambda_{\min}\leftarrow 0, \lambda_{\max}\leftarrow 1, t_0\leftarrow t$;
\WHILE{$\lambda_{\max}-\lambda_{\min} \ge I_q^{t_0}(k_{t})\land I_q^{t_0}(k^*(S_1))$}
\STATE Start a new round $t\leftarrow t+1$;
\STATE Pick $\lambda \leftarrow \rbr{\lambda_{\min}+\lambda_{\max}}/2$ as the middle point; 
\STATE The principal announces scoring  rule $S_{t} = (1-\lambda) S_0 + \lambda S_1$ and obtain the agent's response $k_{t}$;
\IF{$k_{t}=k^*(S_1)$}
    \STATE Update $\lambda_{\max} \leftarrow \lambda$;
\ELSE
    \STATE Update $\lambda_{\min}\leftarrow \lambda$;
\ENDIF
\ENDWHILE
\end{algorithmic}
\caption{Binary Search {BS}($S_0, S_1, k^*(S_1)$, $t$) }\label{alg:BS}
\end{algorithm}
\end{minipage}}
\end{figure}
The binary searching algorithm at step $t$ works on the segment connecting two scoring rules $S_0$ and $S_1$, where the agent's best response under $S_1$ is $k^*(S_1)$. The goal of this binary search is to find the first switching point of the agent's best response from $k_1$ to another action on this segment. 
The binary searching algorithm keeps updating on $\lambda_{\max}$ and $\lambda_{\min}$ as the candidate interval that contains the first switching point. Note that any $\lambda\in[0, 1]$ corresponds to a scoring rule on the segment connecting $S_0, S_1$. Each time, the algorithm deploy a scoring rule corresponding to $\rbr{\lambda_{\min}+\lambda_{\max}}/2$ and the candidate interval is thus cut by half.
In addition, we consider binary searching to a finite depth $m$ such that the searching error respects the minimal uncertainty of $\hat q_{k'}^t, \hat q_{k_t^*}^t$, i.e., 
$2^{-m}\le I_{q}^{t}(k') \land I_{q}^{t}(k_t^*)$ in Algorithm \ref{alg:BS}. Thus, the binary search achieves sufficient accuracy with logarithmic searching time.

\section{Proofs on Action-informed Oracle in \Cref{sec:oracle}}\label{sec:prof oracle}
In this section, we prove the claims in examples in \Cref{sec:oracle}.

\subsection{Proof of \Cref{lem:impossible}}\label{prof:impossible}
Without loss of generality, we just need to construct a hard instance by considering the case $K=3$ and $M=3$. 
To simplify the discussion, we consider two hidden states $\Omega=\{\omega_1, \omega_2\}$ and consider $B_S=1$ as the boxing condition for the scoring rules.
The idea of constructing the hard instance is as follows:
\begin{itemize}
\item[(i)] Make sure that action $a_2$ is the optimal agent's response for the optimal scoring rule while any scoring rule not inducing the agent to chose $a_2$ yields a constant regret at least $1$. 
\item[(ii)] Let $\mathcal{V}_2$ be a "single point" on the boundary of $\mathcal{S}$ and show that with a great chance, any algorithm can fail to find the correct scoring rule located in $\mathcal{V}_2$ without the oracle.
\end{itemize}
Since $\Omega=\{\omega_1, \omega_2\}$, we can equally  represent $\sigma_1, \sigma_2, \sigma_3$ by the their mass assigned to $\omega_1$. With a little abuse of notation, we let $\sigma_i$ denote $\sigma_i(\omega_1)$ and $1-\sigma_i$ denote $\sigma_i(\omega_2)$. 
We consider the case $\Sigma=\{\sigma_1=(0, 1), \sigma_2=(.5, .5), \sigma_3=(1,0)\}$ for $M=3$. 
For simplicity, we let $S_i = \mathbb{E}_\omega S(\sigma_i, \omega)$ for $i=1, 2, 3$.
Hence, the set of proper scoring rules is specified by conditions:
\begin{equation}\label{cond:proper-1}
0\le S_2\le \frac{S_1+S_3}{2}, \quad B_S\ge S_1\ge 0, \quad B_S \ge S_3\ge 0.
\end{equation}
We remark that the first condition guarantees that the scoring rule is proper by requiring the curve of $G(\sigma) = \EE_\omega[S(\sigma, \omega)]$ to be convex \citep{gneiting2007strictly}, and the rest are just box conditions.
In fact, one can easily construct a proper scoring rule $S(\sigma, \omega)$ for any $S=(S_1, S_2, S_3)$ satisfying the above constraints using the Schervish representation \citep{gneiting2007strictly} by fitting a convex function $G$ defined on $[0, 1]$ passing through $(0, S_1), (.5, S_2), (1, S_3)$,  and consider the supporting hyperplanes of $G$ at these three points.

To ensure that $\mathcal{V}_2$ is a single point on the boundary in (ii), we let $q_2- q_1$ and $q_2-q_3$ be
$$q_2-q_1 = \beta \cdot (0, -1, 1)^\top, \quad q_2-q_3=\beta \cdot (1, 1, -2)^\top, $$
where $\beta$ should be considered a fixed constant and $q_1, q_2, a_3$ are what we want to design. These conditions imply the following expression for $\mathcal{V}_2$,
$$
\cV_2 = \{S\in \mathcal{S}\,|\,(0, -1, 1) S \ge \beta^{-1}(c_2-c_1), \quad (1, 1, -2)S\ge \beta^{-1}(c_2-c_3)\}. 
$$
We further let $e_1 = \beta^{-1}(c_2-c_1)$ and $e_2 = \beta^{-1}(c_2-c_3)$. One can easily verify that as long as $e_1+e_2 = 1$ and $-1/2\le e_1\le 0$,
$\mathcal{V}_2$ shrinks to a single point given by $S^*=(1, -e_1, 0)$ and the conditions in Equation \eqref{cond:proper-1} are always satisfied. Hence, we can safely restrict ourselves to the instances specified by parameter $e_1$ and $c_1, c_2, c_3$ can be adjusted accordingly. 
Hence, we are actually considering a single parameter linear system.

To ensure (i), i.e., the optimality of $a_2$, we first need to ensure that $a_2$ beats $a_1$ and $a_2$ at $S^*$ by some constant $\gamma\ge 1$:
\begin{equation}\label{cond:optimality-2}\begin{aligned} 
\langle q_2 - q_1, u - S^*\rangle&\ge \gamma,\\ 
\langle q_2 - q_3, u - S^*\rangle&\ge \gamma, 
\end{aligned}\end{equation}

where $S^*\in \mathcal{V}_2$. Note that since $q_2 - q_1$ does not lie on the same line with $q_2-q_3$, we can always choose $u$ such that \eqref{cond:optimality-2} holds regardless of what $S^*$ is. For any $S\in\mathcal{V}_1$, we can verify that for the principal's profit, 
$$
\langle q_2, u-S^*\rangle - \langle q_1, u-S\rangle = \langle q_2 -q_1, u-S^*\rangle + \langle q_1, S- S^*\rangle \ge \gamma - 1, 
$$
where the last inequality is a direct result of \eqref{cond:optimality-2} and the fact that each entry of $S$ should be within $[0, 1]$. 
The same holds for $S\in\mathcal{V}_3$. 
Hence, it suffices to find an instance satisfying \eqref{cond:optimality-2} with $\gamma=2$. 

Construction of the hard instance:
Following the above discussion, we construct the instance by letting $\gamma=2$, $\beta = 1/16$, and 
\begin{equation}\begin{aligned}
q_1 &= (3/4, 3/16, 1/16), \nonumber\\
q_2 &=(3/4, 1/8, 1/8), \nonumber\\
q_3 &=(11/16, 1/16, 1/4) \nonumber\\
\nonumber u& = (96, 0, 32) + (1, -e_1, 0).\\
\end{aligned}\end{equation}
Here, the only thing unknown is the cost $c_1, c_2, c_3$, where the effect of the unknown cost is encoded in $e_1\in[-1/2, 0]$, which only affects the position of $S^*$.

Since we are considering a one parameter linear system with the second coordinate of $S^*$ undetermined, it is without loss of generality to consider finding the optimal scoring rule on the line $l=(1, \lambda, 0)$ where $0\le \lambda\le 1/2$. Suppose we assign scoring rule $S^t=(1, b^t, 0)$ at step $t$ and the agent responses $a^t\in\mathcal{A}=\{a_1, a_2,a_3\}$. In fact, for any $b^t< e_1$, the agent will response $a_1$ and for any $b^t> e_1$, the agent will response $a_3$. 
Then, we assert that any online strategy of the principal can result in linear regret since no algorithm can decide the exact position of $e_1$ under the feedback of the agent's actions only (Perhaps the best one can do is binary searching for $e_1$, which does not help either). 
In addition, any $b^t\neq e_1$ yields a suboptimality at least 1 at each round. 
Moreover, we remark that letting $\cV_2$ shrink to a point still makes $S^*$ the optimal scoring rule by the tie-breaking rule which is in favor of the principal.
Nevertheless, one can always make $\cV_2$ an interval around $e_1$ as small as possible.
Hence, we conclude that for any algorithm, there always exists a hard instance that yield a linear regret $\text{Reg}\ge \Omega(T)$.
We finish the proof of \Cref{lem:impossible}.

\subsection{Proof of \Cref{exp:random sampling}}\label{sec:random sampling proof}

To ensure the scoring rules we sample are $\epsilon$ better on the induced action, we sample from the strongly proper scoring rules.

\begin{definition}[Strongly proper scoring rules]\label{def:strongly proper}
    A scoring rule $S:\Delta(\Omega)\times\Omega\to\RR $ is $\beta$-strongly proper if for all  $p, q\in \Delta(\Omega)$, $\EE_{\omega\sim q}[S(q, \omega)]-\EE_{\omega\sim q}[S(p, \omega)]\geq \frac{\beta}{2}\nbr{q-p}^2_1$.
\end{definition}

\begin{example}[Restatement of \Cref{exp:random sampling}]
Let  $d_1=\min_{1\le i< j\le M}\nbr{\sigma_i-\sigma_j}_\infty$ and $ d_2=\min_{1\le k<k'\le K}\allowbreak\max_{i\in[M]}\sbr{q_k(i)-q_{k'}(i)}$.  
Since the proper scoring rule class $\cS$ has bounded norm, the $\beta$-strongly proper scoring class also has bounded volume $\vol(\cS_\beta)<\infty$. 
Set $\kappa={d_1^2\beta}/{2}$ and
let $\tilde\cV_k=\{S\in\cS_\beta\given g(k, S)\ge g(k', S)+\kappa, \forall k'\neq k, k'\in[K]\}$. We  suppose $\vol(\tilde\cV_k)\ge \eta \Vol(\cS_\beta)$ for $k\in[K]$. 
Initiate $\cM=\emptyset$ as the candidate set.
Each time, we randomly sample a $\beta$-strongly proper scoring rule $S\in\cS_\beta$ and obtain the agent's best response $k^*(S)$ with respect to $S$. 
Let $e_i(\sigma, \omega)=\ind(\sigma=\sigma_i)$ for $i\in[M]$. By setting $\kappa=d_1^2\cdot\beta/2$, we deploy $\cbr{S-\kappa e_i}_{i\in[M]}$ and see if the agent still responds $k^*(S)$. If so, Let $\sS=\sS\cup \cbr{S}$; if not, reject $S$. 
    After $\cO(M\eta^{-1} K\log K)$ rounds, with high probability, $\sS$  serves as a valid action-informed oracle  with parameter $\epsilon=d_2\cdot \kappa$.
\end{example}


Note that the volumn of $\tilde\cV_k$ is continuous,  decreasing in $\beta$, and goes to $\vol(\tilde\cV_k)$ as $\beta\to 0$. Hence, if $\cV_k$ has a constant fraction of total volumn, we can find $\beta$ such that $\tilde \cV_k$ has a constant fraction of total volumn.

\begin{proof}

First, if $\kappa=d_1^2\cdot\beta/2$, $\cbr{S-\epsilon e_i}_{i\in[M]}$ is also proper since $S$ is $\beta$-strongly proper.

Next, let $d_2=\min_{k'}\max_{i}\sbr{q_k(i)-q_{k'}(i)}$. we show that any each time the probability of hitting a scoring rule in $\tilde\cV_k$ is at least $\eta$. Suppose $S$ is not rejected if for all $i\in [M]$, the agent prefers action $k$ to $k'$: $g(k, S-\kappa e_i)-g(k', S-\kappa e_i)\geq 0$. By properness of $S-\kappa e_i$, 
\begin{align*}
    g(k, S-\kappa e_i)-g(k', S-\kappa e_i)=\langle S-\kappa e_i, q_k-q_{k'}\rangle-(c_k-c_{k'})\\
    =\langle S, q_k-q_{k'}\rangle - \kappa (q_k(i)-q_{k'}(i))-(c_k-c_{k'})\geq 0.
\end{align*}
If this holds for any $i\in [M]$ and $k'\in[K]$, it means $\langle S, q_k-q_{k'}\rangle-(c_k-c_{k'})\geq  \max_{i\in[M]}\kappa (q_k(i)-q_{k'}(i))\geq\kappa d_2$. One the other hand, if $S$ is in set $\tilde \cV_k$, it never gets rejected. 

  After $\frac{c}{\eta}K\log K$ rounds, the probability that one action is not induced is $(1-\eta)^{\frac{c}{\eta}K\log K}=(\frac{1}{K})^{cK}$. Taking a union bound, the probability that any action is not induced is at most $(\frac{1}{K})^{cK-1}$. Setting $c$ large enough, this probability will be sufficiently small. To conclude, after $\cO(M\eta^{-1} K\log K)$ rounds, we get an oracle with parameter $\epsilon=d_2\cdot \kappa$.

\end{proof}

\subsection{Proof of \Cref{exp:linear contract}}
\label{sec:linear contract proof}

\begin{example}[Restatement of \Cref{exp:linear contract}]
Suppose that these $K$ signals are sorted in the increasing order of the cost. Define $u_k = \EE_{\omega\sim\sigma, \sigma\sim q_k}\sbr{u(a^*(\sigma), \omega)}$ and suppose $u_k>0$. We assume the marginal information gain is strictly decaying, i.e., there exists a $\epsilon>0$ such that 
    \begin{gather*}
        \frac{u_{K}-u_{K-1}}{c_{K}-c_{K-1}} >\epsilon, \quad \text{and}\\
        \quad
        \frac{u_k-u_{k-1}}{c_k-c_{k-1}} - \frac{u_{k+1}-u_k}{c_{k+1}-c_k} >\epsilon, \quad \forall k=2,\cdots, K-1.
    \end{gather*}
    Moreover, we assume that $(u_2-u_1)/(c_2-c_1)\le b$. The principal sets the scoring rule as $S(\sigma, \omega)=\lambda u(a^*(\sigma), \omega)$ and conducts binary search on $\lambda\in[0, 2/\epsilon]$. Specifically, the binary searches are iteratively conducted on all the segments $(\lambda_1, \lambda_2)$ with $k^*(\lambda_1 u) \neq k^*(\lambda_2 u)$, where $\lambda_1, \lambda_2$ are neighboring points on $[0, 1/\epsilon]$ that are previously searched. With the maximal searching depth $m= \ceil{\log_2(2b(b-\epsilon)/\epsilon^2)}$, we can identify all the agent's actions. Suppose that $(\lambda_1^k, \lambda_2^k)$ is the largest segment with $\lambda_1^k$ and $\lambda_2^k$ searched before and $k^*(\lambda_1^k u)=k^*(\lambda_2^k u)=k$. By letting $\tilde S_k=(\lambda_1^k + \lambda_2^k)u/2$, we obtain an oracle with $\varepsilon=\epsilon u_1/4b^2$. The procedure takes $\cO(K\log_2(\varepsilon^{-1}))$ rounds.
\end{example}

\begin{proof}
    
First, we show that by setting $\lambda=\lambda_i=\rbr{c_{i}-c_{i-1}}/\rbr{u_i-u_{i-1}}$ for any $i=2,\dots,K$, the agent is indifferent between taking actions $i$ and $i-1$.
In addition, we let $\lambda_1=0$ and $\lambda_{K+1}=2/\epsilon$.
For any $\lambda\in (\lambda_i, \lambda_{i+1})$, the agent best responds with action $b_i$. In addition, $\lambda_2=(c_2-c_1)/(u_2-u_1)\ge 1/b$ and $\lambda_K=(c_K-c_{K-1})/(u_K-u_{K-1})\le 1/\epsilon$. Thus, by searching in $\lambda\in[0, 2/\epsilon]$, all the actions are guaranteed to be induced by this linear contract.

Since the sequence $\rbr{u_i-u_{i-1}}/\rbr{c_i-c_{i-1}}$ is strictly decaying, $\lambda_i$ is strictly increasing and $\lambda_i-\lambda_{i-1}$ is lower bounded by ${1}/\rbr{b-\epsilon}-{1}/{b}={\epsilon}/{b(b-\epsilon)}$, where the lower bound is reached by $\lambda_2=1/b$ and $\lambda_3=1/(b-\epsilon)$. Combined with the fact that search is conducted on $[0, 2/\epsilon]$, we get the maximal searching depth $m= \ceil{\log_2(2b(b-\epsilon)/\epsilon^2)}$ such that the neighboring searching points has a gap no more than $\epsilon/2b(b-\epsilon)$. Therefore, using the fact that $\lambda_i-\lambda_{i-1}\ge \epsilon/b(b-\epsilon)$, we can guarantee that all the actions are induced by this binary search.

For the chosen $\lambda_1^k$ and $\lambda_2^k$ such that $k^*(\lambda_1^k u)=k^*(\lambda_2^k u)=k$, we can guarantee that $\lambda_{k}\le\lambda_1^k\le \lambda_{k} +\epsilon/2b(b-\epsilon)$ and $\lambda_{k+1}-\epsilon/2b(b-\epsilon)\le\lambda_2^k\le \lambda_{k+1}$. Hence, $\lambda_k+\epsilon/4b(b-\epsilon)\le (\lambda_1^k+\lambda_2^2)/2\le \lambda_{k+1}-\epsilon/4b(b-\epsilon)$.
Therefore, by setting $\lambda=(\lambda_1^k + \lambda_2^k)/2$, the marginal profit for choosing $b_k$ over other actions is lower bounded by $\epsilon/4b(b-\epsilon)\cdot u_k\ge \epsilon/4b(b-\epsilon)\cdot u_1\ge \epsilon u_1/4b^2$.
Hence, we obtain the oracle with $\varepsilon=\epsilon u_1/4b^2$.

\end{proof}

\section{Proof of \Cref{thm:regret}}\label{prof: main theorem}
\begin{proof}[Proof of \Cref{thm:regret}]
Here we give a high-level analysis of the $\tilde{\cO}(K^2 T^{-2/3})$ regret for Algorithm~\ref{alg:UCB}. We consider two types of events in $T$ rounds, $\ind(k_t=k_t^*)$ or $\ind(k_t\neq k_t^*)$. In the first type of events, by the construction of $S_t$, the suboptimality gap at each round can be decoupled into components contributed by $(1-\alpha_t)S_t^*$ or $\alpha_t \tilde S_{k^*_t}$. The cumulative regret from $(1-\alpha_t)S_t^*$ follows the optimism principle of UCB and is on the order of $T^{1/2}$. The cumulative regret from $\alpha_t \tilde S_{k^*_t}$ is on the order of $T^{2/3}$, since $\alpha_t = t^{-1/3}$ and the suboptimality gap of $S_{k^*_t}$ is constant. Meanwhile, for the second type of events, we can bound the total number of rounds such event occurs in the worst case as $\tilde{\cO}(K^2 T^{2/3})$. This requires some involved arguments, though it is intuitively the sample complexity to learn $\norm{\hat{q}_k - q_k}\leq t^{-1/3}$ for every $k\in[K]$. This corresponds to the worst case where these signals are competitive and $n_{k, t}$ grows at the same rate for any $k\in[K]$, which means that we have to precisely learn the belief distributions and costs for all the actions.
Therefore, based on the dominating term, the total regret is $\tilde{\cO}(K T^{2/3})$.

We analyze the regret by the following two cases: (i) rounds $t$ such that the agent responds with action $k_t=k_t^*$; (ii) rounds $t$ such that the agent responds with action $k_t\neq k_t^*$. Let $\subopt(k, S_t)=\max_{S\in\cS} h(S)-\EE_{\omega\sim\sigma, \sigma\sim q_k}[u(a^*(\sigma_t), \omega_t) - S_t(\sigma_t, \omega_t)]$ be the regret of implementing scoring rule $S_t$ with the agent selecting action $b_k$ in a single round. We have for the online regret that,
\begin{align}
    \Reg^\pi(T)&=\EE_\pi\sbr{\sum_{t=1}^T \rbr{\max_{S\in\cS} h(S) - \EE_{\omega_t\sim\sigma_t, \sigma_t\sim q_{k_t}}\sbr{u(a^*(\sigma_t), \omega_t) - S_t(\sigma_t, \omega_t)}}} \nonumber\\
&=\sum_{t=1}^T\underbrace{\EE\sbr{\ind(k_t=k_t^*)\subopt(k_t, S_t)}}_{\displaystyle{A_t}}+\sum_{t=1}^T\underbrace{\EE\sbr{\ind(k_t\neq k_t^*)\subopt(k_t, S_t)}}_{\displaystyle{B_t}}. \nonumber
\end{align}
Here and in the sequel, we always use $\EE$ to denote the expectation taken with respect to the data generating process described in \eqref{eq:data generating}.

\paragraph{Bounding regret $A_t$.}
Since $S_t=\alpha_t \Tilde{S}_{k_t^*}+(1-\alpha_t)S_t^*$, it follows from the linearity of $\subopt(k_t, S_t)$ that
\begin{align}
    \subopt(k_t, S_t)=\alpha_t\subopt(k_t, \Tilde{S}_{k_t^*})+(1-\alpha_t)\subopt(k_t, S_t^*).\label{eq:subopt decomp}
\end{align}
Here, the first term $\subopt(k_t, \Tilde{S}_{k_t^*})$ is bounded by constant $C_1=2(B_S+B_u)$, and the second term $\subopt(k_t, S_t^*)$ is bounded by $2\left(\nbr{S_t^*}_\infty+B_u\right) I_{q, k_t^*}^t$ with probability at least $1-1/t$ by the following lemma.
\begin{lemma}\label{lem:linear-program-bound}
Define event $\cE_t$ as $
    \cE_t=\cbr{\nbr{q_k-\hat q_k^t}_1\le I_q^t(k), \forall k\in[K]^+}.
$
Then event $\cE_t$ holds with probability at least $1-1/t$ and it holds on event $\cE_t$ that
\begin{align*}
    u_{k_t^*} - v_{S_t^*}(k_t^*) + 2\left(B_S+B_u\right) I_{q}^t(k_t^*)\ge \max_{S\in\cS} h(S),
\end{align*}
where $h(S)=\EE_{\omega\sim \sigma, \sigma\sim q_{k^*(S)}}\sbr{u(a^*(\sigma), \omega) - S(\sigma, \omega)}$ is the optimal principal's profit, $k_t^* = \argmax_{k\in[K]} h_\LP^t(k)$,  and
$S_t^* := S_{\LP, k_t^*}^t$.
\begin{proof}
    See \S\ref{prof:linear-program-bound} for a detailed proof.
\end{proof}
\end{lemma}
The first term on the right-hand side of \eqref{eq:subopt decomp} can be directly bounded by $\alpha_t C_1$. For the second term on event $\cE_t$, it can be bounded by $C_1 I_q^t(k_t^*)$ using \Cref{lem:linear-program-bound}.
If event $\cE_t$ does not hold, the second term is still bounded by $C_1$ with probability at most $1/t$. Thus, we have for $\alpha_t=K t^{-1/3}\land 1$ that 
\begin{align}
    \sum_{t=1}^T A_t&\leq \sum_{t=1}^T C_1 \rbr{\alpha_t+\frac 1 t+I_{q}^t(k_t^*)}\nonumber\\
    &\leq C_1 \rbr{ K \rbr{\frac 3 2T^{2/3}+1}+(\log T+1)+\sum_{k\in[K]^+}\sum_{t:k_t^*=k} \sqrt{\frac{2\log(K\cdot 2^M t)}{n_{k}^t}}}\nonumber\\
    &\leq C_1\rbr{K \rbr{\frac 3 2T^{2/3}+1}+\log T+1}+C_1\sqrt{2\log(K\cdot 2^M T)}\cdot \sum_{k\in[K]^+}\rbr{2\sqrt{n_k^T}+1} \nonumber\\
    &\leq \cO(T^{2/3})+\cO(\sqrt{|M|+\log(KT)}\cdot\sqrt{KT}), 
\end{align}
where the first equality follows from our previous discussions on \eqref{eq:subopt decomp}, the second inequality holds from the definition of $I_q^t$, and the last inequality holds by the Jensen's inequality that $\sum_{k\in[K]^+}\sqrt{n_k^T}/K\le \sqrt{\sum_{k\in[K]^+}n_k^T/K}$.

\paragraph{Bounding regret $B_t$.} 
We characterize regret $\sum_{t=1}^T B_t$ by bounding the expected number of rounds that $k_t\neq k_t^*$.
To do so, we invoke the following lemma.
\begin{lemma}\label{lem:mistake}
For any fixed $i\neq k_t^*$, under the condition that
\begin{align*}
    \alpha_t\ge 2\varepsilon^{-1} \rbr{I_{c}^{t}(k_t^*, i)+B_S \rbr{I_{q}^t(i)+I_{q}^t(k_t^*)}}:=\Delta_t(k_t^*,i), 
\end{align*}
the agent must respond $k_t\neq i$ under the scoring rule $S_t=\alpha_t \Tilde{S}_{k_t^*}+(1-\alpha_t)S_t^*$ on event $\cE_t$.
\begin{proof}
    See \S\ref{prof:mistake} for a detailed proof.
\end{proof}
\end{lemma}
Under the condition $k_t\neq k_t^*$, the algorithm must be doing binary search.
Now, we introduce the notations that will be used in bounding $B_t$.
Compare to the definition in \Cref{alg:BS}, we instead define $\BS(S_0, S_1, \tilde \lambda_{\min}, \tilde\lambda_{\max}, \tau_0, \tau_1, k_0, k_1, t_0, t_1)$ for a binary search (BS), where $(S_0, S_1)$ is the initial segment that this BS is conducted on, $(\tilde \lambda_{\min}, \tilde\lambda_{\max})$ corresponds to the value of $(\lambda_{\min}, \lambda_{\max})$ that this BS ends with,  $\tau_0, \tau_1$ correspond to the rounds in which $(1-\tilde\lambda_{\min}) S_0 + \tilde\lambda_{\min} S_1$ and $(1-\tilde\lambda_{\max}) S_0 + \tilde\lambda_{\max} S_1$ are played, respectively, $k_0, k_1$ are the best response under $(1-\tilde\lambda_{\min}) S_0 + \tilde\lambda_{\min} S_1$ and $(1-\tilde\lambda_{\max}) S_0 + \tilde\lambda_{\max} S_1$, and $t_0, t_1$ are the starting round and the ending round of this BS.
Notably, $k_0$ is not necessarily the best response under $S_0$ but $k_1$ is guaranteed to be the best response under both $S_1$ and $(1-\tilde\lambda_{\max}) S_0 + \tilde\lambda_{\max} S_1$.

For the BS that lasts until round $t$, we let $t_0(t)$ be the first round of the BS (if round $t$ is not doing BS, we just let $t_0(t)=t$), $(k_0(t), k_1(t))$ be the best response that this BS ends with, $(S_0(t), S_1(t))$ be the segment that this BS searches on.
We consider the following two situations for the case $k_t\neq k_t^*$: (i) $\Delta_{t_0 (t)}(k_0(t), k_1(t))\le \alpha_{t_0(t)}$;
(ii) $\Delta_{t_0 (t)}(k_0(t), k_1(t))> \alpha_{t_0(t)}$.
Following such an idea, we have

    \begin{align*}
        \sum_{t=1}^T B_t&= \sum_{t=1}^T \EE\sbr{\ind(k_t\neq k_t^*)\subopt(k_t, S_t)}\\
        &\le C_1\left[\sum_{t=1}^T\EE\sbr{\ind(\BS=1, \Delta_{t_0 (t)}(k_0(t), k_1(t))\le  \alpha_{t_0(t)})}+\sum_{t=1}^T\EE\sbr{\ind(\BS=1, \Delta_{t_0 (t)}(k_0(t), k_1(t))> \alpha_{t_0(t)})}\right]\\
        &\leq C_1\bigg(\underbrace{\sum_{t=1}^T\frac{1}{t_0(t)}}_{\displaystyle{(B.1)}}+\underbrace{\sum_{t=1}^T\EE\sbr{\ind(\BS=1, \Delta_{t_0 (t)}(k_0(t), k_1(t))> \alpha_{t_0(t)})}}_{\displaystyle{(B.2)}} \bigg),
    \end{align*}
where the first inequality holds by $\subopt(k_t, S_t)\le C_1$, the second inequality holds by \Cref{lem:mistake}.
Specifically, when doing a binary search, the deployed scoring rule $S_t$ must lie on the segment $(S_0(t), S_1(t))$, where $S_0(t)=\alpha_{t_0(t)}\tilde S_{k^*_{t_0(t)}} + (1-\alpha_{t_0(t)}) S_{t_0(t)}^*$ and $S_1(t)=\tilde S_{k^*_{t_0(t)}}$. Hence, $S_t$ can be equivalently expressed as 
\begin{align*}
    S_0(t)=\alpha_{t}'\tilde S_{k^*_{t_0(t)}} + (1-\alpha_{t}') S_{t_0(t)}^*, 
\end{align*}
where $\alpha_{t}'\ge \alpha_{t_0(t)} $.
Therefore, under the condition that $\Delta_{t_0 (t)}(k_0(t), k_1(t))\le \alpha_{t_0(t)}$ and that the agent responds $k_0(t)\neq k_1(t)$ in this binary search, we can use \Cref{lem:mistake} to bound the probability by $1/t_0(t)$. 

To bound $(B.2)$, we define a concept called \emph{essential binary search}.
\begin{definition}[Essential binary search]
We call a binary search $\BS(S_0, S_1, \tilde \lambda_{\min}, \tilde\lambda_{\max}, \tau_0, \tau_1, k_0, k_1, t_0, t_1)$  an \textit{essential binary search} (essential BS) if $\alpha_{t_0}< \Delta_{t_0}(k_0, k_1)$ where $\alpha_{t}=K t^{-1/3}\land 1$.
\end{definition}
The following lemma bounds the total rounds of essential BSs.

    \begin{lemma}\label{lem:essentialBS}
    Consider binary search $\BS(S_0, S_1, \tilde \lambda_{\min}, \tilde\lambda_{\max}, \tau_0, \tau_1, k_0, k_1, t_0, t_1)$. Suppose that the total number of essential binary searches in $T$ rounds is $N$. Then, we have
    \begin{equation}
        N\leq K\cdot \left(12 \varepsilon^{-1}B_S\right)^{2} 2 \log \left(K 2^{M} T\right) T^{2 / 3}.
    \end{equation}
    \begin{proof}
        See \S\ref{prof:essentialBS} for a detailed proof.
    \end{proof}
    \end{lemma}
Let $m$ be the maximal rounds in a binary search, i.e., $t_0(t)\ge t- m$. 
Using the terminal criteria for binary search, we have
\begin{align*}
    m\le \max_{t, k}-\log_2\rbr{I_q^t(k)}=\cO(\log(T)).
\end{align*}
The term $(B.1)$ is directly bounded by
\begin{align*}
    (B.1)\le \log(T) + m=\cO(\log(T)).
\end{align*}
The term $(B.2)$ is bounded using \Cref{lem:essentialBS}, 
\begin{align*}
    (B.2)\le m\cdot K\cdot \left(12 \varepsilon^{-1}B_S\right)^{2} 2 \log \left(K 2^{M} T\right) T^{2 / 3} = \cO(\varepsilon^{-2}B_S^2(M +\log(KT))mKT^{2/3}).
\end{align*}
which finishes the proof.

\end{proof}

\subsection{Proof of \Cref{lem:linear-program-bound}}\label{prof:linear-program-bound}
\begin{lemma}[Restatement of \Cref{lem:linear-program-bound}]
Define event $\cE_t$ as $
    \cE_t=\cbr{\nbr{q_k-\hat q_k^t}_1\le I_q^t(k), \forall k\in[K]^+}.
$
Then event $\cE_t$ holds with probability at least $1-1/t$ and it holds on event $\cE_t$ that
\begin{align*}
    \inp[]{q_{k_t^*}}{u-S_t^*}_\Sigma + 2\left(B_S+B_u\right) I_{q}^t(k_t^*)\ge \max_{S\in\cS} h(S),
\end{align*}
where $h(S)=\EE_{\omega\sim \sigma, \sigma\sim q_{k^*(S)}}\sbr{u(a^*(\sigma), \omega) - S(\sigma, \omega)}$, $v_S(k)=\inp[]{q_k}{S}_\Sigma$ and $k^*(S)$ is the agent's best response.
\end{lemma}
\begin{proof}[Proof of \Cref{lem:linear-program-bound}]
We first state the concentration result from \citet{mardia2020concentration}.
\begin{lemma}[\citet{mardia2020concentration}, Lemma 1, Concentration for empirical distribution]\label{lem:confidence-interval}
Let $p$ be a probability distribution supported in a finite set with cardinality at most $M$ and $p_n$ be the empirical distribution of $n$ i.i.d. samples from $p$. Then, for all sample size $n\in\NN_+$ and $0<\delta<1$, 
\begin{align*}
    \PP\rbr{\nbr{p-p_n}_1\ge \sqrt{\frac{2\log\rbr{2^M/\delta}}{n}}}\le \delta.
\end{align*}
\end{lemma}
By \Cref{lem:confidence-interval} and taking a union bound, with probability at least $1-1/t$, we have $\nbr{\hat q_k^t-q_k}_1\le I_q^t(k)$ for all $k\in[K]^+$, which gives event $\cE_t$.
In the sequel, we discuss under the condition that event $\cE_t$ holds.
We can verify for $v_S(k)=\inp[]{q_k}{S}_\Sigma$ that
\begin{equation}\label{eq:v_S error}
    \abr{\hat v_S^t(k)- v_S(k)}=\abr{\inp[]{\hat q_k^t-q_k}{S}_\Sigma }\le B_S I_q^t(k).
\end{equation}
To study the confidence interval for $\hat C^t(i,j)$, we invoke the following lemma. 
\begin{lemma}\label{lem:C interval}
    Under event $
    \cE_t=\cbr{\nbr{q_k-\hat q_k^t}_1\le I_q^t(k), \forall k\in[K]^+}
$, we have
\begin{align*}
    \abr{\hat C^t(i, j)- C(i,j)} \le I_{c}^t(i, j), \quad \forall (i, j)\in[K].
\end{align*}
\begin{proof}
    A direct corollary of \eqref{eq:v_S error} is
\begin{equation}
\begin{aligned}\label{eq:C error-1}
    C_+^t(i, j) \ge \min_{\tau<t:k_\tau=i} v_{S_\tau}(i) - v_{S_\tau}(j) \ge C(i, j), \\
    C_-^t(i, j) \le \max_{\tau<t:k_\tau=j} v_{S_\tau}(i) - v_{S_\tau}(j) \le C(i, j).
\end{aligned}
\end{equation}
Using the fact in \eqref{eq:C error-1}, we have
\begin{align}\label{eq:C error-2}
    \abr{\hat C^t(i, j)- C(i,j)} 
    &= \abr{\sum_{(i', j')\in l_{ij}} \frac{C_+^t(i', j') + C_-^t(i', j')}{2}-C(i', j')}\nend
    &\le \frac 1 2 \sum_{(i', j')\in l_{ij}} \rbr{\abr{C_+^t(i', j')- C(i', j')} + \abr{C_-^t(i', j')- C(i', j')}}\nend
    &= \frac 1 2 \sum_{(i', j')\in l_{ij}} \abr{C_+^t(i', j')-C_-^t(i', j')} = I_c^t(i, j), 
\end{align}
where the inequality holds by triangular inequality, the second equality holds by \eqref{eq:C error-1}, and the last equality holds by definition of $I_C^t$.
\end{proof}
\end{lemma}

let $k^*$ and $S^*$ be the best response and the optimal scoring rule that achieves the optimal objective $\sup_{S\in\cS} h(S)$ in \eqref{eq:stackelberg-2}, respectively.
Recall the linear programming $\OptLP_k$, 
\begin{equation}
\begin{aligned}
    \OptLP_k:\qquad &\max_{S\in\cS} \quad \inp[]{u}{\hat q_k^t}_\Sigma + B_u I_q^t(k) -v, \nend
    &\mathrm{s.t.}\quad \abr{v-\hat v_{S}^t(k)} \le B_S \cdot I_{q}^t(k), \\
    &\qquad \quad v - \hat v_{S}^t(i) \ge \hat C^t(k, i) - \rbr{I_{c}^{t}(k, i) + B_S\cdot I_{q}^t(i)}, \quad \forall i\neq k.
\end{aligned}
\end{equation}
It is easy to verify that $S^*$ satisfies all the constraint in $\OptLP_{k^*}$ using \eqref{eq:v_S error} and \eqref{eq:C error-2} with $v=v_{S^*}(k^*)$ and $k=k^*$.
Also, $\inp[]{u}{\hat q_k^t}_\Sigma + B_u I_q^t(k)\ge \inp[]{u}{q_k}_\Sigma$ for $\forall k\in[K]^+$.
Therefore, by noting that $k_t^* = \argmax_{k\in[K]^+} h_\LP^t(k)$, we conclude that $h_\LP^t(k_t^*)\ge h_\LP^t(k^*)\ge h(S^*)$.
Also, let $v_\LP^*$ be the solution to $\OptLP_{k_t^*}$ and recall that $S_t^*=S_{\LP, k_t^*}^{t}$ we have
\begin{align*}
    h_\LP^t(k_t^*)&=\inp[]{u}{\hat q_{k_t^*}^t}_\Sigma + B_u I_q^t(k_t^*) -v_\LP^*\nend
    &\le (B_u+B_S) I^t_q(k^*)+\EE_{\omega\sim \sigma, \sigma\sim \hat{q}_{k^*_t}}\sbr{u(a^*(\sigma), \omega) - S^*_t(\sigma, \omega)}\\
    &\le 2(B_u +B_S )I^t_q(k^*) +\EE_{\omega\sim \sigma, \sigma\sim q_{k^*_t}}\sbr{u(a^*(\sigma), \omega)-S_t^*(\sigma, \omega)}, 
\end{align*}
where the first inequality holds by noting that $\abr{v_\LP^*-\hat v_{S_t^*}(k_t^*)}\le B_S I_{q}^t(k_t^*)$, and the second inequality holds by definition of event $\cE_t$.
Thus, we finish the proof. 
\end{proof}

\subsection{Proof of \Cref{lem:mistake}}\label{prof:mistake}
\begin{lemma}[Restatement of \Cref{lem:mistake}]
For $\forall i\neq k_t^*$, under the condition that
\begin{align*}
    \alpha_t\ge 2\varepsilon^{-1} \rbr{I_{c}^{t}(k_t^*, i)+B_S \rbr{I_{q}^t(i)+I_{q}^t(k_t^*)}}:=\Delta_t(k_t^*,i), 
\end{align*}
the agent will not respond $k_t=i$ under the scoring rule $S_t=\alpha_t \Tilde{S}_{k_t^*}+(1-\alpha_t)S_t^*$ on event $\cE_t$.
\end{lemma}
\begin{proof}[Proof of \Cref{lem:mistake}]
On event $\cE_t$,  for any $i\neq k_t^*$, if $\alpha_t\ge \Delta_t(k_t^*,i)$, we aim to show the agent prefers action $b_{k_t^*}$ to action $b_i$. Recall $S_t=\alpha_t \Tilde{S}_{k_t^*}+(1-\alpha_t)S_t^*$. The expected profit for the agent generated by responding action $b_i$ is
\begin{equation*}
    v_{S_t}(i)-c_i
    \leq\alpha_tv_{\Tilde{S}_{k_t^*}}(i)+(1-\alpha_t)\hat{v}^t_{S_t^*}(i)-c_i+(1-\alpha_t)B_SI_q^t(i),
\end{equation*}
and the expected utility generated by responding action $b_{k_t^*}$ is
\begin{equation*}
     v_{S_t}(k_t^*)-c_{k_t^*}
    \geq\alpha_tv_{\Tilde{S}_{k_t^*}}(k_t^*)+(1-\alpha_t)\hat{v}^t_{S_t^*}(k_t^*)-c_{k_t^*}-(1-\alpha_t)B_SI_q^t(k_t^*).
\end{equation*}
By \Cref{asp:oracle}, we already have $(v_{\Tilde{S}_{k_t^*}}(k_t^*)-c_{k_t^*})-(v_{\Tilde{S}_{k_t^*}}(i)-c_{i})>\varepsilon$. 
Since $S_{t}^*$ is the solution to $\OptLP_{k_t^*}$, following the constraints of $\OptLP_{k_t^*}$ we have $$\hat{v}^t_{S_t^*}(k_t^*)-\hat{v}^t_{S_t^*}(i)\geq \hat{C}^t(k_t^*, i)-I^t_c(k_t^*, i)-B_S (I^t_q(i)+I^t_q(k_t^*)).$$ 
Combining the above inequalities, we get
\begin{align*}
     \sbr{v_{S_t}(k_t^*)-c_{k_t^*}}-\sbr{ v_{S_t}(i)-c_i}
     &\ge \alpha_t\cdot \varepsilon +(1-\alpha_t)\rbr{\hat{C}^t(k_t^*, i)-C(k_t^*, i)-I^t_c(k_t^*, i)-B_S (I^t_q(i)+I^t_q(k_t^*))}\nend
     &\qquad - (1-\alpha_t) B_S\rbr{I_q^t(k_t^*)+I_q^t(i)}\nend
     &\geq \alpha_t\cdot\varepsilon+2(1-\alpha_t)\sbr{-I^t_c(k_t^*, i)-B_S (I^t_q(i)+I^t_q(k_t^*))}\\
     &>\alpha_t \cdot \varepsilon -2\sbr{I^t_c(k_t^*, i)+B_S (I^t_q(i)+I^t_q(k_t^*))}, 
\end{align*}
where the second inequality holds by recalling that we have proved $C(k_t^*,i)\leq \hat{C}^t(k_t^*, i)+I^t_c(k_t^*, i)$ in \eqref{eq:C error-2} under $\cE_t$, and the last inequality holds by noting that $0<\alpha_t\le 1$.
When $\alpha_{t} \geq 2\varepsilon^{-1}\left(I_{c}^{t}\left(k_{t}^{*}, i\right)+B_S\left(I_{q}^{t}(i)+I_{q}^{t}(k_{t}^{*})\right)\right)$, the above $\sbr{v_{S_t}(k_t^*)-c_{k_t^*}}-\sbr{ v_{S_t}(i)-c_i}>0$, which means the agent prefers action $k_t^*$ to $i$.

    \end{proof}

\subsection{Proof of \Cref{lem:essentialBS}}\label{prof:essentialBS}
In the sequel, 
\Cref{lem:confidence-after-binary-search} shows essential BS reduces confidence interval, while \Cref{lem:no-more-search} states the stopping criteria of essential BS. Combining \Cref{lem:confidence-after-binary-search} and \Cref{lem:no-more-search}, we can bound the total rounds in essential BS.

\begin{proposition}
\label{lem:confidence-after-binary-search}
    For a binary search $\BS(S_0, S_1, \tilde \lambda_{\min}, \tilde\lambda_{\max}, \tau_0, \tau_1, k_0, k_1, t_0, t_1)$, we have
    \begin{equation}
        I_c^t(k_0, k_1)\leq 2\sqrt{2\log(K\cdot 2^M T)}B_S\rbr{\frac{1}{\sqrt{n_{k_0}^{{t_1}}}}+\frac{1}{\sqrt{n_{k_1}^{{t_1}}}}}.
    \end{equation}
\end{proposition}

\begin{proof}
[Proof of \Cref{lem:confidence-after-binary-search}]
By definition of $I_c^t(k_0, k_1)$, for any $t> {t_1}$,
\begin{align*}
    I_c^t(k_0, k_1) = \sum_{(i', j')\in l_{k_0k_1}} \varphi^t(i', j')\le \varphi^t(k_0, k_1),
\end{align*}
where the inequality holds by noting that $l_{k_0k_1}$ is the \say{shortest path} with $\varphi^t$ being the \say{length} of each edge. Using the definition of $\varphi^t$, we conclude that
\begin{align*}
    I_c^t(k_0, k_1)&\le \frac{1}{2}\sbr{C_+^t(k_0, k_1)-C_-^t(k_0, k_1)}_+\\
    &=\frac{1}{2}\sbr{\min_{\tau'<t:k_{\tau'}=k_0}\rbr{\hat{v}^t_{S_{\tau'}}(k_0)-\hat{v}^t_{S_{\tau'}}(k_1)}-\max_{\tau'<t:k_{\tau'}=k_0}\rbr{\hat{v}^t_{S_{\tau'}}(k_0)-\hat{v}^t_{S_{\tau'}}(k_1)}}_+\\
    &\qquad\qquad+B_S\left(I_{q}^{t}(k_0)+I_{q}^{t}(k_1)\right)\\
    &\leq \frac{1}{2}\sbr{\hat{v}^t_{S_{\tau_0}}(k_0)-\hat{v}^t_{S_{\tau_0}}(k_1)-\hat{v}^t_{S_{\tau_1}}(k_0)+\hat{v}^t_{S_{\tau_1}}(k_1)}_+ 
    +B_S\left(I_{q}^{t}(k_0)+I_{q}^{t}(k_1)\right), 
\end{align*}
where the last inequality holds by noting that $t>{t_1}\ge (\tau_0, \tau_1)$. Consequently, we have by the definition of $\hat v_S^t$ that,
\begin{align*}
    I_c^t(k_0, k_1)&\le\frac{1}{2}\abr{\langle S_{\tau_0}- S_{\tau_1}, \hat{q}^t_{k_0}-\hat{q}^t_{k_1}\rangle_\Sigma}+B_S\left(I_{q}^{t}(k_0)+I_{q}^{t}(k_1)\right)\\
    &=\frac{(\tilde\lambda_{\max}-\tilde\lambda_{\min})}{2}\abr{\langle  S_1-S_0, \hat{q}^t_{k_0}-\hat{q}^t_{k_1}\rangle_\Sigma}+B_S\left(I_{q}^{t}(k_0)+I_{q}^{t}(k_1)\right)\\
    &\leq \frac{\min\rbr{I^{t_1}_q(k_0), I^{t_1}_q(k_1)}}{2} \rbr{\abr{\inp[]{S_1-S_0}{\hat q_{k_0}^t}_\Sigma}+\abr{\inp[]{S_1-S_0}{\hat q_{k_1}^t}_\Sigma}}+B_S\left(I_{q}^{t}(k_0)+I_{q}^{t}(k_1)\right), 
\end{align*}
where $\tilde\lambda_{\max}$ and $\tilde\lambda_{\min}$ are the values of $\lambda_{\max}$ and $\lambda_{\min}$ when this BS terminates, respectively. Here, the last inequality holds by the terminal criteria of the BS. Furthermore, we can upper bound the right-hand side by
\begin{align*}
    I_c^t(k_0, k_1)
    &\le \min\rbr{I^{t_1}_q(k_0), I^{t_1}_q(k_1)} B_S+B_S\left(I_{q}^{t}(k_0)+I_{q}^{t}(k_1)\right)\nend
    &\leq \sbr{\sqrt{2\log(K\cdot 2^M T)}\cdot \rbr{\frac{1}{\sqrt{n_{k_0}^{t_1}}}+\frac{1}{\sqrt{n_{k_1}^{t_1}}}+\frac{1}{\sqrt{\max\cbr{n_{k_0}^{t_1}, n_{k_1}^{t_1}}}}}}B_S\\
    &\leq 2B_S\sbr{\sqrt{2\log(K\cdot 2^M T)}\cdot \rbr{\frac{1}{\sqrt{n_{k_0}^{t_1}}}+\frac{1}{\sqrt{n_{k_1}^{t_1}}}}}, 
\end{align*}
where the first inequality holds by noting that $\hat q_{k}^t$ is an empirical probability and the fact that $\nbr{S_1-S_0}_\infty\le B_S$ since $S$ is nonnegative, and the second inequality holds by noting that $n_k^t$ is nondecreasing and ${t_1}<t$.
    
\end{proof}

\begin{proposition}
    \label{lem:no-more-search}
    For a binary search $\BS(S_0, S_1, \tilde \lambda_{\min}, \tilde\lambda_{\max}, \tau_0, \tau_1, k_0, k_1, t_0, t_1)$, if
    \begin{equation}
        \min\{n_{k_0}^{{t_1}}, n_{k_1}^{{t_1}}\}\geq \left(12 \varepsilon^{-1}B_S\right)^{2} 2 \log \left(K 2^{M} T\right)T^{2 / 3}\eqdef N^*,
    \end{equation}
    then we have $t^{-1/3}\ge \Delta_t(k_0, k_1)$ for any $t_1<t\le T$ and 
    no more essential {BS} ending with the same $(k_0, k_1)$ is possible.
\end{proposition}
\begin{proof}
    [Proof of \Cref{lem:no-more-search}]
    For any ${t_1}<t\leq T$, we have
    \begin{align*}
        t^{-1/3}\ge T^{-1/3}&=\left(12 \varepsilon^{-1}B_S\right)\sqrt{\frac{ 2 \log \left(K 2^{M} T\right)}{N^*}}\\
        &\geq 2\left(\varepsilon^{-1}B_S\right)\sqrt{ 2 \log \left(K 2^{M} T\right)}\rbr{2\frac{1}{\sqrt{n^{t_1}_{k_0}}}+2\frac{1}{\sqrt{n^{t_1}_{k_1}}}+\frac{1}{\sqrt{n_{k_0}^t}}+\frac{1}{\sqrt{n_{k_1}^t}}}\\
        &\geq 2\varepsilon^{-1}\rbr{B_S\rbr{I_q^t(k_0)+I_q^t(k_1)} +I_c^t(k_0, k_1)}= \Delta_t(k_0, k_1), 
    \end{align*}
    where the equality holds by definition of $N^*$, the second inequality holds by our condition $\min\{n_{k_0}^{{t_1}}, n_{k_1}^{{t_1}}\}\geq N^*$, and the last inequality holds by using \Cref{lem:confidence-after-binary-search}.
    Therefore, we conclude that by letting $\alpha_t=K t^{-1/3}\land 1\ge t^{-1/3}$, the condition for an essential BS, i.e.,  $\alpha_t< \Delta_t(k_0, k_1)$ no longer holds for any $t$ such that ${t_1}<t\le T$.
\end{proof}
Now, we are ready to prove \Cref{lem:essentialBS}.
\begin{lemma}[Restatement of \Cref{lem:essentialBS}]
    Consider binary search $\BS(S_0, S_1, \tilde \lambda_{\min}, \tilde\lambda_{\max}, \tau_0, \tau_1, k_0, k_1, t_0, t_1)$. Suppose that the total number of essential binary searches in $T$ rounds is $N$. Then, we have
    \begin{equation}
        N\leq \left(12 \varepsilon^{-1}B_S\right)^{2} 2 \log \left(K 2^{M} T\right)K T^{2 / 3}=KN^*.
    \end{equation}
    \end{lemma}
\begin{proof}[Proof of \Cref{lem:essentialBS}]
    To prove \Cref{lem:essentialBS}, we first introduce another concept called \emph{critical binary search} (critical BS).
    \begin{definition}[Critical binary search]
        For an \emph{essential} binary search $\BS(S_0, S_1, \tilde \lambda_{\min}, \tilde\lambda_{\max}, \tau_0, \tau_1, k_0, k_1, t_0, t_1)$, we consider it to be critical if $\min\{n_{k_0}^{t_1}, n_{k_1}^{t_1}\}\ge N^*$.
    \end{definition}
    \paragraph{Critical BSs.}
    We claim that the number of critical BS is no more than $K$.
    Suppose the opposite holds, i.e., number of critical BS no less than $K+1$. 
    Consider a graph $\cG=(V, E)$ where $V=[K]^*$ and $E$ is formed by the set the critical BSs ending with $(k_0, k_1)\in V$.
    We claim that one of the following two cases must hold
    \begin{itemize}
        \item[(i)] There exists $(i,j)\in[K]^+$ such that there are at least two critical BSs ending with $(k_0, k_1)=(i, j)$;
        \item[(ii)] There is a loop in $\cG$ formed by the critical BSs.
    \end{itemize}
    We prove that both (i) and (ii) must not happen. 
    For (i), following from the definition of critical BS that $\min\{n_{k_0}^{{t_1}}, n_{k_1}^{{t_1}}\}\geq N^*$ where $t_1$ is the ending round of a critical BS and using \Cref{lem:no-more-search} that there should be no more essential BSs ending with $(k_0, k_1)$ after $t_1$, we directly have (i) impossible. For (ii), suppose that $\BS(S_0, S_1, \tilde \lambda_{\min}, \tilde\lambda_{\max}, \tau_0, \tau_1, k_0, k_1, t_0, t_1)$ is the first that creates a loop. Then before $t_0$, there must be a path $\ell'_{k_0 k_1}$ of critical BSs. For any edge $(k_0', k_1')$ on $\ell'_{k_0k_1}$, there must be a critical Binary search $\BS(S_0', S_1', \tilde \lambda_{\min}', \tilde\lambda_{\max}', \tau_0', \tau_1', k_0', k_1', t_0', t_1')$ with $t_1'< t_0$. Using \Cref{lem:no-more-search} again, we have
    \begin{align*}
        t_0^{-1/3} \ge \Delta_{t_0}(k_0', k_1'), \quad \forall (k_0', k_1')\in \ell'_{k_0, k_1}.
    \end{align*}
    Since a path in $\cG$ has length at most $K$, we conclude that
    \begin{align*}
        \alpha_{t_0}= K t_0^{-1/3} &\ge \sum_{(k_0', k_1')\in\ell'_{k_0 k_1}} \Delta_{t_0}(k_0', k_1')\\
        &=\sum_{(k_0', k_1')\in\ell'_{k_0 k_1}} 2\varepsilon^{-1}\rbr{B_S\rbr{I_q^{t_0}(k_0')+I_q^{t_0}(k_1')} +I_c^{t_0}(k_0', k_1')}\\
        &\ge 2\varepsilon^{-1}\rbr{B_S\rbr{I_q^{t_0}(k_0)+I_q^{t_0}(k_1)} +I_c^{t_0}(k_0, k_1)} = \Delta_{t_0}(k_0, k_1), 
    \end{align*}
    where the last inequality holds by noting that $I_c^{t_0}(k_0', k_1')\le \varphi^{t_0}(k_0', k_1')$ and $I_c^{t_0}$ is obtained by the shortest path algorithm with edge length $\varphi^{t_0}$. Such a fact again suggests that there cannot be any essential BR ending with $(k_0, k_1)$ that starts at $t_0$. Thus, a loop also cannot be formed, and we prove our claim that the total number of critical BSs is no more than $K$.
    
    \paragraph{Essential but noncritical BSs.}
    Next, we count the number of BSs that are essential but not critical. Recall the graph $\cG=(V, E)$ where $V=[K]^*$ and $E$ is formed by the set the critical BSs ending with $(k_0, k_1)\in V$. 
    Let $w^t(i,j)$ be the total number of essential BSs ending with $(k_0, k_1)=(i,j)$ before round $t$ and let $w(i,j)=w^T(i,j)$. 
    Let $d_k^t=\sum_{k'\neq k} w(k, k')$ be the total weights for node $k\in V$. For any pair $(k_0, k_1)\in V$, if $d^t_{k_1}\land d^t_{k_2}\ge N^*$, there should be no more essential BSs after $t$ by \Cref{lem:no-more-search}. On the other hand, any BS ending with $(k_0, k_1)$ guarantees that $d_{k_0}^t$ and $d_{k_1}^t$ increase at least by $1$. 
    Hence, the problem of counting the number of essential but noncritical BSs can be described as the following weight-placing game, 
    \begin{itemize}
        \item[(i)] initiate $d_k=0$ for $\forall k\in V$;
        \item[(ii)] at each round, select an edge $(k_0,k_1)\in V$ such that $d_{k_0}\land d_{k_1}< N^*$ and add $1$ to both $d_{k_0}$ and $d_{k_1}$;
        \item[(iii)] the game ends if no more edge can be selected.
    \end{itemize}
    The total rounds of this weight-placing game upper bound the total number of essential but noncritical BSs.
    We have the following proposition that gives the maximal number of rounds in this weight-placing game.
    \begin{proposition}[Weight-placing game]
        For the previously described weight-placing game, the total number of rounds should not exceed $(|V|-1)(N^*-1)$.
        \begin{proof}
            We prove this proposition by induction on $|V|$. The case $|V|=1$ is trivial. For $|V|=2$ the result is obvious since there is only one edge to choose. Suppose that the result holds for $|V|\le v-1$. For $|V|=v$, suppose that $(v-1)(N^*-1)+1$ rounds are played. We study the quantity $d^-=\min_{k\in V} d_k$.
\begin{itemize}
    \item If there exists $k\in V$ such that $d_k\le N^*-2$, then for the subgraph $\cG^{k}=\cG\backslash\{k\}$, we have at least $(v-2)(N^*-1)+2$ rounds played on $\cG^k$. However, since $\cG^k$ only has $v-1$ nodes, we have from the induction that there should be no more than $(v-2)(N^*-1)$ rounds played on $\cG^k$, which causes a contradictory. Thus, we have $d^-\ge N^*-1$ if $(v-1)(N^*-1)+1$ rounds are played.
    \item If $d^-\ge N^*$, then the last round in this play should be mistakenly played following rules (ii) and (iii). 
\end{itemize}
Therefore, we conclude that $d^-=N^*-1$ and we are able to choose a node $k\in V$ such that $d_k=N^*-1$. On the subgraph $\cG^k=\cG\backslash\{k\}$, we have by induction that at most $(v-2)(N^*-1)$ rounds can be played. Therefore, the total rounds played on $\cG$ is no more than
\begin{align*}
    (v-2)(N^*-1)+d_k = (v-1)(N^*-1), 
\end{align*}
which contradicts our assumption that  $(v-1)(N^*-1)+1$ rounds are played. 
Thus, for $|V|=v$, no more than $(v-1)(N^*-1)$ rounds can be played.
By induction, we conclude that the proposition holds for any $|V|$.
        \end{proof}
    \end{proposition}
    Hence, the total number of essential but noncritical BSs is bounded by $K(N^*-1)$. Given that the number of critical BSs bounded by $K$, the total number of essential BSs is thus bounded by $KN^*$.
\end{proof}

\end{document}